\documentclass{article}
\usepackage[utf8]{inputenc}
\usepackage{amsmath,amsthm, amssymb, amsfonts}
\usepackage{todonotes}
\usepackage{kantlipsum}
\usepackage{wrapfig}
\usepackage{paralist}
\usepackage{wasysym}
\usepackage{verbatim}
\usepackage{fullpage}
\usepackage{hyperref}
\usepackage{graphicx}
\usepackage{wrapfig}
\usepackage{mathtools}
\usepackage{mathrsfs}
\usepackage{enumitem}
\usepackage{dsfont}
\usepackage{units}
\usepackage[]{algorithm2e}
\usepackage{ltablex}
\usepackage{varwidth}
\usepackage{caption}
\usepackage{changepage}
\usepackage{subcaption}

\usepackage{xcolor}

\allowdisplaybreaks[1]
\usepackage{sistyle} 
\newcommand\bSI[1]{{\small[\SI{}{#1}]}}

\makeatletter
\newlength\unitwdth
\newlength\numwdth
\newlength\tdima
\newcommand\SIdescr[2]{%
    \setlength\tdima{\linewidth}%
    \addtolength\tdima{\@totalleftmargin}%
    \addtolength\tdima{-\dimen\@curtab}%
    \addtolength\tdima{-\unitwdth}%
    \addtolength\tdima{-\numwdth}%
    \parbox[t]{\tdima}{%
        #1
        \leaders\hbox{$\m@th\mkern \@dotsep mu\hbox{\tiny.}\mkern \@dotsep mu$}%
        \hfill
        \ifhmode\strut\fi
        \makebox[0pt][l]{%
            \makebox[\unitwdth][l]{}%
            \makebox[\numwdth][r]{#2}}}}
\makeatother

\newcommand{\R}{\mathbb{R}}
\newcommand{\N}{\mathbb{N}}

\newcommand{\Z}{\mathbb{Z}}

\newcommand{\bmat}[2]{\left[ \begin{array}{#1} #2 \end{array} \right]}

\renewcommand{\epsilon}{\varepsilon}
\newcommand{\eps}{\varepsilon}

\renewcommand{\rho}{\varrho}

\newcommand{\Arch}{{\mathcal{A}}}

\newcommand{\RNcal}{\mathcal{RN}}
\newcommand{\Ycal}{\mathcal{Y}}
\newcommand{\Ccal}{\mathcal{C}}

\newcommand{\act}[1]{R_\sigma\left(#1\right)}

\DeclareMathOperator{\argmax}{arg\,max}

\DeclareMathOperator{\vcdim}{VCdim}

\DeclarePairedDelimiter{\ceil}{\lceil}{\rceil}
\DeclarePairedDelimiter{\floor}{\lfloor}{\rfloor}
\DeclarePairedDelimiterX{\norm}[1]{\lVert}{\rVert}{#1}
\DeclarePairedDelimiterX{\pabs}[1]{\lvert}{\rvert}{#1}

\newcommand{\sNN}[3]{\RNcal_{(#1,#2,#3), \sigma}}

\usepackage{tikz}
\usepackage{pgfplots}
\usepackage{tkz-euclide}
\usepackage[edges]{forest}
\usetikzlibrary{math}
\usepgfplotslibrary{external} 
\usepackage{pgfplots}
\usepgfplotslibrary{colormaps}
\usepackage{pgfplots,pgfplotstable}
\pgfplotsset{compat=1.5}
\usepackage{tikz-3dplot}
\makeatletter
\renewcommand{\todo}[2][]{\tikzexternaldisable\@todo[#1]{#2}\tikzexternalenable}
\makeatother
\tikzset{
    regular/.style = {shape=circle, draw, align=center, color=black, fill=blue,font=\tiny},
}

\newtheorem{theorem}{Theorem}[section]
\newtheorem*{theorem*}{Theorem}
\newtheorem{remark}[theorem]{Remark}
\newtheorem{definition}[theorem]{Definition}

\newtheorem{corollary}[theorem]{Corollary}

\newtheorem*{remark*}{Remark}
\newtheorem*{proposition*}{Proposition}
\newtheorem{example}[theorem]{Example}
\numberwithin{equation}{section}

\definecolor{darkcandyapplered}{rgb}{0.64, 0.0, 0.0}

\title{Expressivity of Deep Neural Networks\footnote{This review paper will appear as a book chapter in the book ``Theory of Deep Learning'' by Cambridge University Press. 
}}
\author{Ingo G\"uhring \thanks{These authors contributed equally.} \thanks{Institute of Mathematics, Technical University of Berlin,  Stra\ss{}e des 17.~Juni 136, 10623 Berlin, Germany; E-Mail: {$\{$\texttt{guehring, raslan, kutyniok}$\}$\texttt{@math.tu-berlin.de}}} 
\and Mones Raslan \footnotemark[2] \footnotemark[3]  
\and Gitta Kutyniok\footnotemark[3] \thanks{Faculty - Electrical Engineering and Computer Science, Technical University of Berlin
 / Department of Physics and Technology, University of Troms\o}}

\usepackage[ddmmyyyy,hhmmss]{datetime}
\newdateformat{daymonthyeardate}{%
  \THEDAY.\THEMONTH.\THEYEAR}
\usepackage{fancyhdr}
\usepackage{lastpage}
\usepackage{etoolbox}
\begin{document}
\maketitle

\noindent

\begin{abstract}
In this review paper, we give a comprehensive overview of the large variety of approximation results for neural networks. Approximation rates for classical function spaces as well as benefits of deep neural networks over shallow ones for specifically structured function classes are discussed. While the main body of existing results is for general feedforward architectures, we also review approximation results for convolutional, residual and recurrent neural networks. 
\end{abstract}

\textbf{Keywords:} {approximation, expressivity, deep vs. shallow, function classes}

\def\quality{60}
\def\b{2}

\section{Introduction }\label{sec:Intro}
While many aspects of the success of deep learning still lack a comprehensive mathematical explanation, the approximation properties\footnote{Throughout the paper, we will interchangeably use the term \emph{approximation} theory and \emph{expressivity} theory.} of neural networks have been studied since around 1960 and are relatively well understood. \emph{Statistical learning theory} formalizes the problem of approximating - in this context also called \emph{learning} - a function from a finite set of samples. Next to statistical and algorithmic considerations, approximation theory plays a major role for the analysis of statistical learning problems. We will clarify this in the following by introducing some fundamental notions.\footnote{\cite{cucker2007learning} provides a concise introduction to statistical learning theory from the point of view of approximation theory.}

Assume that $\mathcal{X}$ is an \emph{input space} and $\mathcal{Y}$ a \emph{target space}, $\mathcal{L}:\mathcal{Y}\times \mathcal{Y} \to [0,\infty]$ is a \emph{loss function} and $\mathbb{P}_{(\mathcal{X},\mathcal{Y})}$ a (usually unknown) \emph{probability distribution} on some $\sigma$-algebra of $\mathcal{X}\times \mathcal{Y}$. We then aim at finding a minimizer of the risk functional\footnote{With the convention that $\mathcal{R}(f)=\infty$ if the integral is not well-defined.}
$$
    \mathcal{R}:\mathcal{Y}^{\mathcal{X}}\to [0,\infty],~f\mapsto  \int_{\mathcal{X}\times \mathcal{Y}} \mathcal{L}\left(f(x),y\right) \mathrm{d}\mathbb{P}_{(\mathcal{X},\mathcal{Y})}(x,y),
$$ 
induced by $\mathcal{L}$ and $\mathbb{P}_{(\mathcal{X},\mathcal{Y})}$ (where $\mathcal{Y}^{\mathcal{X}}$ denotes the set of all functions from $\mathcal{X}$ to $\mathcal{Y}$). That means we are looking for a function $\hat f$ with
$$ 
    \hat{f} = \mathrm{argmin} \left\{\mathcal{R}(f):~f\in \mathcal{Y}^\mathcal{X}  \right\}.
$$ 
In the overwhelming majority of practical applications, however, this optimization problem turns out to be infeasible due to three reasons:

\begin{enumerate}
    \item[(i)] The set $\mathcal{Y}^{\mathcal{X}}$ is simply too large, such that one usually fixes a priori some \emph{hypothesis class} $\mathcal{H}\subset \mathcal{Y}^{\mathcal{X}}$ and instead searches for
$$
    \hat{f}_{\mathcal{H}} = \mathrm{argmin} \left\{ \mathcal{R}(f) : ~f\in \mathcal{H} \right\}.
$$
In the context of deep learning, the set $\mathcal{H}$ consists of \emph{deep neural networks}, which we will introduce later. 
\item[(ii)] Since $\mathbb{P}_{(\mathcal{X},\mathcal{Y})}$ is unknown, one can not compute the risk of a given function $f$. Instead, we are given a \emph{training set} $\mathcal{S}=((x_i,y_i))_{i=1}^m$, which consists of $m\in \N$ i.i.d. samples drawn from $\mathcal{X}\times\mathcal{Y}$ with respect to $\mathbb{P}_{(\mathcal{X},\mathcal{Y})}$. Thus, we can only hope to find the  minimizer of the \emph{empirical risk} $\mathcal{R}_\mathcal{S}(f) = \frac{1}{m}\sum_{i=1}^m \mathcal{L}(f(x_i),y_i))$ given by
$$
    \hat{f}_{\mathcal{H},\mathcal{S}} = \mathrm{argmin} \left\{ \mathcal{R}_{\mathcal{S}}(f):~f\in \mathcal{H} \right\}.
$$
\item[(iii)] In the case of deep learning one needs to solve a complicated non-convex optimization problem to find $\hat{f}_{\mathcal{H},\mathcal{S}}$, which is called \emph{training} and can only be done approximately. 
\end{enumerate}

Denoting by $\hat{f}^*_{\mathcal{H},\mathcal{S}}\in \mathcal{H}$ the approximative solution, the overall error 
\begin{align*}
 \big|\mathcal{R}\big(\hat{f}\big)-\mathcal{R}\big(\hat{f}^*_{\mathcal{H},\mathcal{S}}\big)\big|\leq &\underbrace{\big|\mathcal{R}\big(\hat{f}^*_{\mathcal{H},\mathcal{S}}\big)- \mathcal{R}\big(\hat{f}_{\mathcal{H},\mathcal{S}} \big)\big|}_{training~ error}\\
 +&\underbrace{\big| \mathcal{R}\big(\hat{f}_{\mathcal{H},\mathcal{S}} \big) - \mathcal{R}\big(\hat{f}_{\mathcal{H}} \big)\big|}_{estimation~ error}\\
 +&\underbrace{\big|  \mathcal{R}\big(\hat{f}_{\mathcal{H}} \big)-\mathcal{R}\big(\hat{f} \big)\big|}_{approximation~ error}.
\end{align*}
\centering
\begin{figure}

\centering
\begin{tikzpicture}[scale=.8, transform shape]
\tikzmath{ 
    \unit = .7;
    \bx1 = 18*\unit;
    \by1 = 8*\unit;
    \tx1 = 0*\unit;
    \ty1 = 0*\unit;
}
\tikzmath{ 
    \bmargx = 0.5*\unit;
    \bmargy = 0.5*\unit;
    \tmargy = 2*\unit;
    \bx2 = \bx1*2/3;
    \by2 = \by1-\tmargy;
    \tx2 = \tx1+\bmargx;
    \ty2 = \ty1+\bmargy;
    \t = 12*\unit;
}
\newcommand{\anklrec}[6]{
    \pgfsetcornersarced{\pgfpoint{40*\unit}{40*\unit}}
    \draw [draw=black,fill=#6] (#1,#2) rectangle (#3,#4);
    \node at (#3/2,#4-.5) {#5};
}
\newcommand{\setfont}[1]{
\fontsize{#1pt}{#1pt}\selectfont
}
    \anklrec{\tx1}{\ty1}{\bx1}{\by1}{$\mathcal{Y}^{\mathcal{X}}$}{white} 
    \anklrec{\tx2}{\ty2}{\bx2}{\by2}{$\mathcal{H}$}{blue!20}
    \node[circle, fill=black,label=$\hat{f}^*_{\mathcal{H},\mathcal{S}}$] at (\tx2+\bx2/10,\ty2+\by2*2/5) (true)  {} ;
    \node[circle, fill=black,label=$\hat{f}_{\mathcal{H},\mathcal{S}}$] at (\tx2+\bx2/2,\ty2+\by2*9/20)   (truetrain){} ;
    \node[circle, fill=black,label={[xshift=-8*\unit,label distance=1*\unit]93:$\hat{f}_{\mathcal{H}}$}] at (\tx2+\bx2-.6*\unit,\ty2+\by2*5/8)  (trueestimate) {} ;
    \draw[<->,line width = 3*\unit] (true.east) -- (truetrain.west) node[midway, below, align=left,label={[xshift=0*\unit, yshift=-50*\unit,align=left]{\setfont{\t}\textbf{Training Error}\\ (Optimization)}}]{};
    \draw[<->,line width = 3*\unit] (truetrain.east) -- (trueestimate.west) node[midway, below, align=left,label={[xshift=10*\unit, yshift=-50*\unit,align=left]{\setfont{\t}\textbf{Estimation Error}\\ (Generalization)}}]{};
    
    \node[circle, fill=black,label=$\hat{f}$] at (\tx1+\bx1*9/10,\ty1+\by1*14/20)   (trueapprox){} ;
   \draw[<->,line width = 3*\unit](trueestimate)--(trueapprox) node[midway, below, align=left,label={[xshift=30*\unit, yshift=-50*\unit,align=left]{\setfont{\t}\textbf{Approximation Error}\\ (Expressivity)}}]{};
\end{tikzpicture}

\caption{Decomposition of the overall error into training error, estimation error and approximation error}
\label{fig:ApproxError}

\end{figure}
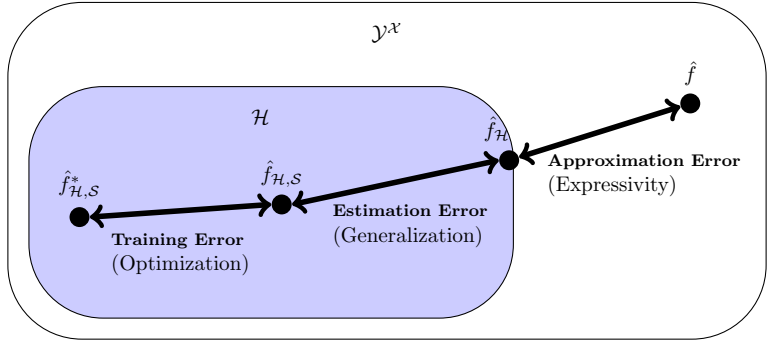
\flushleft

The results discussed in this paper deal with estimating the approximation error if the set $\mathcal{H}$ consists of deep neural networks. However, the observant reader will notice that practically all of the results presented below ignore the dependence on the unknown probability distribution $\mathbb{P}_{(\mathcal{X},\mathcal{Y})}$. This can be justified by different strategies (see also \cite{cucker2007learning}) from which we will depict one here. 

Under suitable conditions it is possible to bound the approximation error by 
$$
    \big|\mathcal{R}\big(\hat{f}_{\mathcal{H}} \big)-\mathcal{R}\big(\hat{f} \big) \big|  \leq  \mathrm{error}\big(\hat{f}_\mathcal{H}-\hat{f}\big),
$$
where $\mathrm{error}(\cdot)$ is an expression (e.g. the $\|\cdot\|_\infty$ norm) that is independent of $\mathbb{P}_{(\mathcal{X},\mathcal{Y})}$. As an example, assume that $\mathcal{Y}\subset \R$, and the loss function $\mathcal{L}(\cdot,y)$ is Lipschitz continuous for all $y\in \Ycal$ with uniform Lipschitz constant $\mathrm{Lip}(\mathcal{L})$. We then get
\begin{equation}\label{eq:approx_error_bound}
    \big|\mathcal{R}\big(\hat{f}_{\mathcal{H}} \big)-\mathcal{R}\big(\hat{f} \big) \big| \leq \mathrm{Lip}(\mathcal{L}) \cdot \big\|\hat{f}_{\mathcal{H}}- \hat{f}\big\|_{\infty},
\end{equation}
and hence an upper bound of $\|\hat{f}_{\mathcal{H}}- \hat{f}\|_{\infty}$ can be used to upper bound the approximation error. 

The \emph{universal approximation theorem} (see \cite{FUNAHASHI1989,Cybenko1989,Hornik1989universalApprox,hornik1991approximation}), which is the starting point of approximation theory of neural networks, states:
 
\vspace{.75em}
\begin{adjustwidth}{3em}{3em}
\emph{For every $\hat{f}\in C(K)$ with $K\subset \R^d$ compact and every $\epsilon>0$ there exists a neural network $\hat{f}_{\mathcal{H},\eps}$ such that $\|\hat{f}_{\mathcal{H}}-\hat{f} \|_{\infty}\leq \epsilon$.}
\end{adjustwidth}
\vspace{.75em}

\noindent Utilizing that neural networks are universal approximators, we can now see from Equation~\eqref{eq:approx_error_bound} that for $\mathcal{H}=C(K)$ the approximation error can be made arbitrarily small. In practice, we are faced with a finite memory and computation budget, which shows the importance of results similar to the theorem above that additionally quantify the complexity of $\hat{f}_{\mathcal{H}}$.

We now proceed by introducing the notion of neural networks considered throughout this paper in mathematical terms.

\subsection{Neural Networks}
 We now give a mathematical definition of feedforward neural networks, which were first introduced in \cite{MP43}. More refined architectures like convolutional, residual and recurrent neural networks are defined in the corresponding sections (see Subsection~\ref{sec:SpecialArchitecture}).

In most cases it makes the exposition simpler to differentiate between a neural network as a collection of weights and biases, and the corresponding function, referred to as its realization.\footnote{However, if it is clear from the context, in the following we denote by neural networks both the parameter collections as well as their corresponding realizations.} The following notion has been introduced in \cite{PetV2018OptApproxReLU}.

\begin{definition}\label{def:NeuralNetworks}
Let $d,s, L\in \N$.
A \emph{neural network $\Phi$ with input dimension $d,$ output dimension $s$ and $L$ layers}
is a sequence of matrix-vector tuples
\[
  \Phi = \big( (\mathbf{W}^{[1]},\mathbf{b}^{[1]}), (\mathbf{W}^{[2]},\mathbf{b}^{[2]}), \dots, (\mathbf{W}^{[L]},\mathbf{b}^{[L]}) \big),
\]
where $n_0 = d,~n_L=s$ and $n_1, \dots, n_{L-1} \in \N$, and where each
$\mathbf{W}^{[\ell]}$ is an $n_{\ell} \times n_{\ell-1}$ matrix,
and $\mathbf{b}^{[\ell]} \in \R^{n_\ell}$.

If $\Phi$ is a neural network as above, $K \subset \R^d$, and if
$\sigma\colon \R \to \R$ is arbitrary, then we define the associated
\emph{realization of $\Phi$ with activation function $\sigma$ over $K$}
(in short, the $\sigma$\emph{-realization of $\Phi$ over $K$})
as the map $R_{\sigma}(\Phi)\colon K \to \R^{s}$ such that
\[
  R_{\sigma}(\Phi)(x) = \mathbf{x}^{[L]} ,
\]
where $\mathbf{x}^{[L]}$ results from the following scheme:
\begin{equation*}
  \begin{split}
    \mathbf{x}^{[0]} &\coloneqq {x}, \\
    \mathbf{x}^{[\ell]} &\coloneqq \sigma(\mathbf{W}^{[\ell]} \, \mathbf{x}^{[\ell-1]} + \mathbf{b}^{[\ell]}),
    \qquad \text{ for } \ell = 1, \dots, L-1,\\
    \mathbf{x}^{[L]} &\coloneqq \mathbf{W}^{[L]} \, \mathbf{x}^{[L-1]} + \mathbf{b}^{[L]},
  \end{split}
\end{equation*}
and where $\sigma$ acts componentwise, that is,
$\sigma({v}) = (\sigma({v}_1), \dots, \sigma({v}_m))$
for every ${v} = ({v}_1, \dots, {v}_m) \in \R^m$.

We call $N(\Phi) \coloneqq d + \sum_{j = 1}^L n_j$ the
\emph{number of neurons of the neural network} $\Phi$ and $L = L(\Phi)$ the
\emph{number of layers}. For $\ell \leq L$ we call $M_\ell(\Phi) \coloneqq \|\mathbf{W}^{[\ell]}\|_0 + \|\mathbf{b}^{[\ell]}\|_0$ the \emph{number of weights in the $\ell$-th layer} and we define $M(\Phi)\coloneqq \sum_{\ell = 1}^L M_\ell(\Phi)$, which we call the \emph{number of weights of $\Phi.$} Finally, we denote by $\max\{N_1,\dots,N_{L-1}\}$ the \emph{width} of $\Phi.$
\end{definition}

\begin{figure}[ht!]
\centering
    \includegraphics[width=0.45\textwidth]{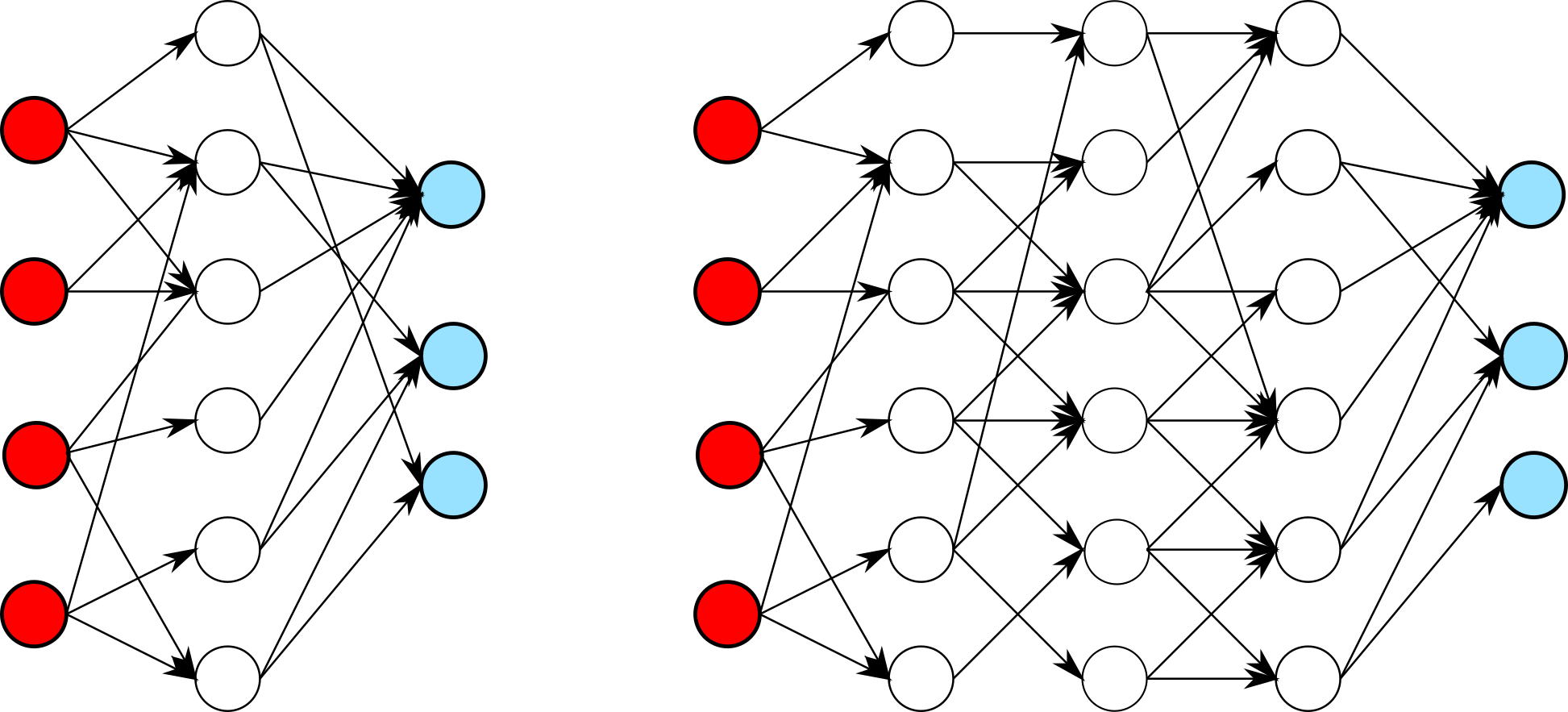}
    \caption{Visualization of a (\textbf{left}) shallow and a (\textbf{right}) deep (feedforward) neural network with input dimension $d=4$ and output dimension $s=3$.}
\end{figure}

Although the activation can be chosen arbitrarily, a variety of particularly useful activation functions has been used in the context of deep learning. We refer to Table \ref{tab:ActFunctions}, which is an adapted version of \cite[Table 1]{TopNet}, for an overview of frequently used activation functions.

\begin{centering}
\begin{footnotesize}
\renewcommand{\arraystretch}{1.5}
\begin{table}[ht!]
 \centering
\begin{tabular}{|c|c|}
\hline
  \textbf{Name}
& \textbf{Given by}
\\ \hline
  rectified linear unit (ReLU)
& $\max\{0,x\}$

\\ \hline
  $a$-leaky ReLU
& $\max\{ax,x\}$ for some $a \geq 0$, $a \neq 1$

\\ \hline
  exponential linear unit
& $x\cdot \chi_{x\geq 0}(x)+ (\exp(x)-1)\cdot \chi_{x<0}(x)$

\\ \hline
  softsign
& $\frac{x}{1+|x|}$

\\ \hline
  \vtop{\hbox{\strut $a$-inverse square root}\hbox{\strut linear unit }}
& $x\cdot \chi_{x\geq 0}(x) + \frac{x}{\sqrt{1+ax^2}}\cdot \chi_{x<0}(x)$
  for $a>0$

\\ \hline
  $a$-inverse square root unit
& $\frac{x}{\sqrt{1 + a x^2}}$ for $a>0$

\\ \hline 
sigmoidal (type) & $\lim_{x\to \infty} \sigma(x)=1$ and $\lim_{x\to -\infty} \sigma(x) =0.$

\\ \hline
  sigmoid / logistic
& $\frac{1}{1+\exp(-x)}$

\\ \hline
  tanh
& $\frac{\exp(x)-\exp(-x)}{\exp(x)+\exp(-x)}$

\\ \hline
  arctan
& $\arctan(x)$

\\ \hline
  softplus
& $\ln(1+\exp(x))$
\\ \hline

\end{tabular}
\caption{Commonly-used activation functions. }
 \label{tab:ActFunctions}
\end{table}
\end{footnotesize}
\end{centering}

Many results give a relation between the approximation accuracy and the \emph{complexity} of a neural network $\Phi$, which is measured in terms of the number of neurons $N(\Phi),$ the number of non-zero weights and biases $M(\Phi)$ and the number of layers $L(\Phi).$ 

Before we proceed, let us fix the following notions concerning the set of all  (realizations of) neural networks.

\begin{definition}
    Let $d=n_0,~n_1,\dots,n_{L-1},~s=n_L\in \N$ for some $L\in \N,$ and $\sigma:\R\to \R.$ Then we set 
    $$
        \mathcal{N}_{(d,n_1,\dots,n_{L-1},s)} \coloneqq \left\{\Phi \text{ neural network with } L \text{ layers, }  n_\ell \text{ neurons in layer } \ell \right\}
    $$
    as well as 
    $$
        \mathcal{RN}_{(d,n_1,\dots,n_{L-1},s),\sigma} \coloneqq \left\{R_{\sigma}(\Phi):\Phi\in \mathcal{N}_{(d,n_1,\dots,n_{L-1},s)} \right\}.
    $$

    In the following, if not stated otherwise, we fix an input dimension $d\in \N$ as well as an output dimension $s\in \N.$
     
\end{definition}

\subsection{Goal and Outline}

The aim of this paper is to provide a comprehensive overview of the area of approximation theory for neural networks. In Subsection~\ref{sec:Shallow}, we start with the universal approximation theorem for \emph{shallow} neural networks. These were the main types of neural networks studied until the 1990's. We then relate the approximation accuracy of neural networks to their complexity. It turns out that for many well-known classical function spaces one can derive upper and lower complexity bounds in terms of the number of weights, neurons and layers. 

Since 2012, \emph{deep} neural networks have shown remarkable success in various fields of applications. In Subsection~\ref{sec:Universal}, we transfer universality results for shallow neural networks to their deep counterparts. We then proceed with approximation rates of deep neural networks for classes of smooth functions in Subsection~\ref{sec:ApproxSmooth} and piecewise smooth functions in Subsection~\ref{sec:DeepPiecewiseRates}. Function classes for that the so-called \emph{curse of dimensionality} can be overcome will be discussed in Subsection~\ref{sec:Structure}. Another important line of research tries to explain the benefit of \emph{deep} neural networks over \emph{shallow} ones. This will be the focus of Subsection~\ref{sec:ComparisonDeepShallow}. 

Finally, there exists a large variety of different network architectures each adapted to specific applications. It remains an interesting question how architectural design choices influence expressivity. In Subsection~\ref{sec:SpecialArchitecture}, we will present results that cover this question for \emph{convolutional} neural networks, which have shown tremendous successes in computer vision, \emph{residual} neural networks, which allowed the use of much deeper models and, lastly, \emph{recurrent} neural networks, which can be viewed as dynamical systems.

\subsection{Notation}

We denote by $\N=\{1,2,\dots\}$ the set of all \emph{natural numbers} and define $\N_0\coloneqq \N\cup \{0\}.$ For $a\in \R$ we set $\lfloor a \rfloor\coloneqq \max\{b\in \Z\colon~b\leq a\} $ and $\lceil a\rceil\coloneqq \min\{b\in \Z\colon~b\geq a\}.$ For two real-valued functions $f,g$ we say that $f\lesssim g$ if there exists some constant $C>0$ such that $f\leq Cg.$ Conversely, we write $f\gtrsim g$ if $g\lesssim f$ and $f\sim g$ if $f\lesssim gb$ and $f\gtrsim g.$ \\

For two sets $A,B$ such that $A\subset B$ we denote by $\mathbf{1}_{A}$ be the \emph{indicator function of $A$ in $B$}. Moreover $|A|$ denotes the \emph{cardinality} of $A$. If $(B,\mathcal{T})$ is a topological space, we denote by $\partial A$ the \emph{boundary of $A$} and by $\overline{A}$ its closure.
 
For $x\in \R^n$ we denote by $|x|$ its \emph{Euclidean norm}, and by $\|x\|_p$ its $p$-norm, $p\in [1,\infty].$ If $(V,\|\cdot\|_V)$ is a normed vector space,
 we denote by $(V^*,\|\cdot\|_{V^*})$ the \emph{topological dual space of $V$,}
 i.e. the set of all scalar-valued, linear, continuous functions equipped with the \emph{operator norm}. 
 
 Let $d,s\in \N.$ For a measurable set $K\subset \R^d$ we denote by $C^n(K,\R^s),~n\in \N_0\cup\{\infty\},$ the spaces of \emph{$n$ times continuously differentiable functions with values in $\R^s$}. Equipped with $\|f\|_{C^n} =\max_{\|\alpha\|_1\leq n} \|D^{\alpha}f\|_{\infty}$ these spaces are Banach spaces if $K$ is compact. We denote by $C_0^{\infty}(K)$ the \emph{set of infinitely many times differentiable functions with compact support in $K$.} In the case that $s=1,$ we simply write $C(K)\coloneqq C(K,\R).$ Let $\beta=(n,\zeta)$ for some $n\in \N_0,~\zeta\in(0,1],~K\subset \R^d$ be compact. Then, for $f\in C^n(K)$, we denote 
 $$
    \|f\|_{C^{\beta}}\coloneqq \max\left\{ \max_{\|\alpha\|_1\leq n} \|D^{\alpha}f\|_\infty,\max_{\|\alpha\|_1=n}\mathrm{Lip}_{\zeta}(D^\alpha f) \right\} \in [0,\infty],
 $$
 where $\mathrm{Lip}_{\zeta}(f)= \sup_{x,y\in K,~x\neq y} \frac{|f(x)-f(y)|}{|x-y|^{\zeta}}.$
 
We denote by $C^{\beta}(K)\coloneqq \left\{ f\in C^{n}(K):\|f\|_{C^{\beta}}<\infty\right\}$ the \emph{space of all $\beta$-H\"older continuous functions.}
 For an $n-$times differentiable function $f:K\subset \R\to\R$ we denote by $f^{(n)}$ its $n$-th derivative.
 
 For a measure space $(K,\mathcal{G},\mu)$ we denote by $\mathcal{L}_p(K;\mu),~p\in [1,\infty]$ the spaces of equivalence classes of $\mathcal{G}$-measurable, real-valued functions $f:K\to \R$ which coincide $\mu$-almost everywhere and for which $\|f\|_p<\infty,$ where
$$
  \|f\|_p\coloneqq   \begin{cases} \left( \int_K |f(x)|^p d\mu(x) \right)^{1/p},~& \text { if } p<\infty, \\ \mathrm{ess~sup}_{x\in K} |f(x)|, & \text { if } p=\infty.\end{cases} 
$$      
If $\lambda$ is the Lebesgue measure on the Lebesgue $\sigma$-algebra of $K\subset \R^d,$ then we will simply write $\mathcal{L}_p(K) = \mathcal{L}_p(K;\lambda)$ as well as $dx = d\lambda(x).$

Let $W^{n,p}(K)$ be the \emph{Sobolev space} of order $n$ consisting of $f\in \mathcal{L}_p(K)$ satisfying $D^{\alpha}f \in \mathcal{L}_p(K)$ for all multi-indices $\|\alpha\|_1\leq n$, where $D^{\alpha}f \in \mathcal{L}_p(K)$ denots the weak derivative. Finally, we denote by $F^{p}_{n,d}$ the \emph{unit ball in $W^{n,p}([0,1]^d)$.}

\section{Shallow Neural Networks}\label{sec:Shallow}

In this section, we examine expressivity results for \emph{shallow} neural networks, which form the groundwork for a variety of results connected to deep neural networks. After reviewing their universality properties in Section \ref{subsec:ShallowUniversal}, we examine lower complexity  bounds in Section \ref{subsec:shallow_lower_bounds} and upper complexity bounds in Section \ref{subsec:shallow_upper_bounds}.

\subsection{Universality of Shallow Neural Networks} \label{subsec:ShallowUniversal}
The most famous types of expressivity results for neural networks state that shallow neural networks are \emph{universal approximators}. This means, that for a wide variety of relevant function classes $\mathcal{C}$, every function $f\in\mathcal{C}$ can be arbitrarily well approximated by a shallow neural network. In mathematical terms the statement reads as follows.
\vspace{.5em}
\begin{adjustwidth}{3em}{3em}
\emph{For every $f\in \Ccal$ and every $\epsilon>0$ as well as different types of activation functions $\sigma:\R\to \R,$ there exists some $N\in \N$ and some neural network $\Phi_{f,\epsilon} \in \mathcal{N}_{(d,N,s)}$ such that 
\begin{align*}
    \left\|f-\act{\Phi_{f,\epsilon}}\right\|_{\Ccal} \leq \epsilon.
\end{align*}}
\end{adjustwidth}
\noindent For all commonly used activation functions $\sigma$ and many relevant function classes $\Ccal,$ the widths $N$ of the approximating neural networks $\Phi_{f,\epsilon}$ do not remain uniformly bounded over $\Ccal$ and $\epsilon$, but grow with increasing approximation accuracy. 

Within a very short period of time three papers containing results in this direction appeared. The first one, \cite{FUNAHASHI1989}, establishes universality of shallow neural networks with non-constant, bounded, and monotonically increasing continuous activation function for the space $C(K),~K\subset \R^d$ compact. The idea of the proof is based on Fourier theory, Paley-Wiener theory, and an integral formula from~\cite{Irie1988}.

A slightly different set of activation functions (monotonically increasing, sigmoidal) was considered in~\cite{Hornik1989universalApprox}, where universality is established for $\Ccal=C(K)$ and $\Ccal = \mathcal{L}_p(\R^d;\mu),$. There, $\mu$ is a probability measure defined on the Borel $\sigma$-algebra of $\R^d$. The main idea behind the proof is based on the Stone-Weierstrass Theorem.

Shortly afterwards, universality of continuous sigmoidal functions has been proved in \cite{Cybenko1989} for the function space $\Ccal = C(K),$ where $K\subset \R^d$ is compact. The proof, whose main ideas we will sketch in Theorem \ref{thm:cyb}, is based on an elegant application of the Hahn-Banach Theorem combined with the measure theoretic version of the Riesz Representation Theorem. 

An extension of this result for discrete choices of scaling weights and biases has been given in \cite{CHUI1992131}. We note that the statements given in~\cite{FUNAHASHI1989,Hornik1989universalApprox,Cybenko1989,CHUI1992131} are all applicable to sigmoidal activation functions, which were commonly used at that time in practice. The result of Cybenko considers the more general activation function class of \emph{discriminatory functions} (see Definition \ref{def:Discriminatory}) with which he is able to establish universality of shallow neural networks for $\Ccal=\mathcal{L}_1(K).$ Moreover, universality of shallow neural networks with a sigmoidal activation function for $\Ccal = \mathcal{L}_2(K)$ based on so-called \emph{Fourier networks} has been shown in~\cite{Gallant1988} and, closely related to this result, in~\cite{hechtnielsen1989}. Another universal result for the space $\Ccal=\mathcal{L}_2(K)$ for continuous, sigmoidal activation functions employing the Radon Transform has been given in~\cite{carroll1989construction}. In~\cite{hornikStinchcombe1990derivative, hornik1991approximation} universality for functions with high-order derivatives is examined. In this case $\mathcal{C}$ is given by the Sobolev space $W^{n,p}(K)$ or the space $C^{n}(K)$ and $\sigma$ is a sufficiently smooth function. 

Further advances under milder conditions on the activation function were made in~\cite{PinkusUniversalApproximation}. Again, their result is based on an application of the Stone-Weierstra{\ss} Theorem. The precise statement along with the main proof idea are depicted in Theorem \ref{thm:PinkusUniversal}.

In the following we present a selection of elegant proof strategies for universal approximation theorems. We start by outlining a proof strategy utilizing the Riesz Representation Theorem for measures (see \cite[Theorem 6.19]{Rudin}).

\begin{definition}[\cite{Cybenko1989}] \label{def:Discriminatory}
Let $K \subset \R^d$ be compact.
A measurable function $f: \R \to \R$ is \emph{discriminatory with respect to} $K,$
if for every finite, signed, regular Borel measure $\mu$ on $K$ we have that
\[
  \Big(
    \int_K f(\mathbf{W} {x} + \mathbf{b}) d\mu({x}) = 0,
    \text{ for all } \mathbf{W} \in \R^{1 \times d} \text{ and } \mathbf{b} \in \R
  \Big)
  \quad \Longrightarrow \quad
  \mu = 0 \,.
\]
\end{definition}

In fact, in \cite[Lemma 1]{Cybenko1989} it has been demonstrated that every sigmoidal function is indeed discriminatory with respect to closed cubes in $\R^d$. The universal approximation result for discriminatory functions now reads as follows.

\begin{theorem}[\cite{Cybenko1989}] \label{thm:cyb}
 Let $\sigma \in C(\R)$ be discriminatory with respect to a compact set $K\subset \R^d$. Then $\mathcal{RN}_{(d,\infty,1),\sigma}\coloneqq \cup_{N\in\N} \mathcal{RN}_{(d,N,1),\sigma}$ is dense in $C(K).$ 
\end{theorem}

\begin{proof}
We restrict ourselves to the case $s=1.$ Towards a contradiction, assume that the linear subspace $\mathcal{RN}_{(d,\infty,1),\sigma}$ is not dense in $C(K).$ Set $\overline{\mathcal{RN}}\coloneqq \overline{\mathcal{RN}_{(d,\infty,1),\sigma}}.$ Then there exists some $f\in C(K)\setminus \overline{\mathcal{RN}}.$ By the Hahn-Banach-Theorem, there exists some $\kappa\in C(K)^{*}$ with $\kappa(f)\neq 0$ and $\kappa|_{\overline{\mathcal{RN}}}=0.$ Invoking the Riesz Representation Theorem (see \cite[Theorem 6.19]{Rudin}), there exists some finite, non-zero, signed Borel measure $\mu$ such that 
    \begin{align*}
        \kappa(f) = \int_K f({x}) d\mu({x}), \quad \text{ for all } f\in C(K).
    \end{align*}
    
Notice that, for all $\mathbf{W}\in \R^{1\times d}$ and all $\mathbf{b}\in \R,$ we have that $\sigma(\mathbf{W} (\cdot) + \mathbf{b}) \in \mathcal{RN}_{(d,\infty,1),\sigma}.$ This implies that 
    \begin{align*}
        0= \kappa\left(\sigma(\mathbf{W}(\cdot) +\mathbf{b})\right) = \int_K \sigma(\mathbf{W}x +\mathbf{b}) d\mu({x}), \quad \text{ for all } \mathbf{W}\in \R^{1\times d}, \mathbf{b}\in \R.
    \end{align*}
    But since $\sigma$ is a discriminatory function, $\mu=0,$ which is a contradiction.
\end{proof}

Now we proceed by depicting the universality result given in \cite{PinkusUniversalApproximation}, which is based on an application of the Stone-Weierstrass Theorem.

\begin{theorem}[\cite{PinkusUniversalApproximation}] \label{thm:PinkusUniversal}
Let $K\subset\R^d$ be a compact set and $\sigma:\R\to\R$ be continuous and not a polynomial. Then $\sNN{d}{\infty}{s}$ is dense in $C(K)$. 
\end{theorem}

\emph{Sketch of Proof:} 
We only consider the case $s=1.$ Moreover, by \cite[Proof of Theorem 1, Step 2]{PinkusUniversalApproximation}, we can restrict ourselves to the case $d=1$. This follows from the fact that if $\mathcal{RN}_{(1,\infty,1),\sigma}$ is dense in $C(\tilde{K})$ for all compact $\tilde{K}\subset \R,$ then $\mathcal{RN}_{(d,\infty,1),\sigma}$ is dense in $C(K)$ for all compact $K\subset\R^d.$

In the following, we will write that $f\in \overline{M}^{C(\R)}$ for some $M\subset C(\R)$ if for every compact set $K\subset \R$ and every $\epsilon >0$ there exists some $g \in M$ such that $\left\| f|_K- g|_K\right\|_\infty \leq \epsilon.$ Hence, the claim follows if we can show that $\overline{ \mathcal{RN}_{(1,\infty,1),\sigma}}^{C(\R)}=C(\R).$

\textbf{Step 1 (Activation $\sigma \in C^{\infty}(\R)$):} Assume that $\sigma \in C^{\infty}(\R).$ Since for every $\mathbf{W},\mathbf{b}\in \R,~h\in \R\setminus\{0\},$
$$
    \frac{\sigma((\mathbf{W}+ h)\cdot+ \mathbf{b})-\sigma(\mathbf{W}\cdot+ \mathbf{b}) }{h} \in \mathcal{RN}_{(1,\infty,1)} ,
$$
 we obtain that $\frac{d}{d\mathbf{W}}\sigma(\mathbf{W}\cdot + \mathbf{b})\in \overline{ \mathcal{RN}_{(1,\infty,1)}}^{C(\R)}.$ By an inductive argument, we obtain $\frac{d^k}{d\mathbf{W}^k}\sigma(\mathbf{W}\cdot + \mathbf{b})\in \overline{ \mathcal{RN}_{(1,\infty,1)}}^{C(\R)}$ for every $k\in \N_0$. Moreover, $\frac{d^k}{d\mathbf{W}^k}\sigma(\mathbf{W}\cdot+\mathbf{b}) = (\cdot)^k \sigma^{(k)}(\mathbf{W}\cdot + \mathbf{b}).$ Since $\sigma$ is not a polynomial, for every $k\in \N_0,$ there exists some $\mathbf{b}_k\in \R$ such that $\sigma^{(k)}(\mathbf{b}_k)\neq 0.$  Hence, $(\cdot)^k \cdot \sigma^{(k)}(\mathbf{b}_k) \in \overline{ \mathcal{RN}_{(1,\infty,1)}}^{C(\R)} \setminus \{0\}$ for all $k\in \N_0$ and $\overline{ \mathcal{RN}_{(1,\infty,1),\sigma}}^{C(\R)}$ contains all monomials and hence also all polynomials. Since polynomials are dense in $C(\R)$ by the Weierstrass Theorem, $\overline{ \mathcal{RN}_{(1,\infty,1),\sigma}}^{C(\R)}=C(\R).$

\textbf{Step 2 (Activation $\sigma \in C(\R)$):} Now assume, that $\sigma \in C(\R)$ and $\sigma$ is not a polynomial. Step 1 is used via a mollification argument by showing that, for $z \in C^\infty_0(\R)$, 
$$\sigma \ast z = \int_{\R} \sigma(\cdot-y)z(y)~dy \in \overline{ \mathcal{RN}_{(1,\infty,1)}}^{C(\R)},$$
holds by an approximation of the integral by Riemann series. If $\sigma \ast z$ is not a polynomial, then by invoking Step 1 we conclude that 
$$
    \overline{\mathcal{RN}_{(1,\infty,1),\sigma}}^{C(\R)}=C(\R).
$$

Finally, using standard arguments from functional analysis, it is shown that for every $\sigma \in C(\R)$ with $\sigma$ not being a polynomial, there exists some $z \in C^\infty_0(\R)$ such that $\sigma \ast z$ is not a polynomial. This yields the claim.

\hfill $\square$

As mentioned before, these universality results do not yield an estimate of the necessary width of a neural network in order to achieve a certain approximation accuracy. However, due to hardware induced constraints on the network size, such an analysis is imperative. 

We will see in the following, that many of the subsequent results suffer from the infamous \emph{curse of dimensionality} \cite{bellman1952theory}, i.e. the number of parameters of the approximating networks grows exponentially in the input dimension.
To be more precise, for a variety of function classes $ \Ccal \subset \{f:\R^d\to \R^{s}\},$ in order to obtain 
\[\left\|f-\act{\Phi_{f,\epsilon}} \right\|_{\Ccal} \leq \epsilon,\] for an unspecified $f\in\Ccal,$ the width of $\Phi_{f,\epsilon}$ needs to scale asymptotically like $\epsilon^{-d/C}$ for a constant $C=C(\Ccal)>0$ as $\epsilon \searrow 0.$ In other words, the complexity of the involved networks grows exponentially in the input dimension with increasing approximation accuracy. 

\subsection{Lower Complexity Bounds}\label{subsec:shallow_lower_bounds}
In this subsection, we aim at answering the following question: Given a function space $\Ccal$ and an approximation accuracy $\eps$, how many (unspecified) weights and neurons are necessary for a neural network such that its realizations are potentially able to achieve approximation accuracy $\epsilon$ for an arbitrary function $f\in \Ccal$?
We start by presenting results for classical function spaces $\Ccal$ where the curse of dimensionality can in general not be avoided. 

The first lower bounds have been deduced by a combination of two arguments in~\cite{maiorov1999,Maiorov1999LowerBounds} for the case where no restrictions on the parameter selection process are imposed. It is also shown in~\cite{maiorov1999} that the set of functions for that this lower bound is attained is of large measure.

\begin{theorem}[\cite{maiorov1999,Maiorov1999LowerBounds}] \label{thm:lowShallow}
Let $\sigma\in C(\R)$, $d\geq 2,$ $\eps>0$ and $N\in\N$ such that, for each $f\in F_{n,d}^2$, there exists a neural network $\Phi_{f,\epsilon}\in\mathcal{N}_{(d,N,s)}$ satisfying
\[
\norm{\act{\Phi_{f,\eps}}- f}_{2}\leq \eps.
\]
Then $N\gtrsim \eps^{-\frac{d-1}{n}}.$ 
\end{theorem}

The next theorem can be used to derive lower bounds if the parameter selection is required to depend continuously on the function to be approximated. We will state the theorem in full generality and draw the connection to neural networks afterwards.
\begin{theorem}[\cite{DeVore1989}] \label{thm:lower_devore}
Let $\eps>0$ and $1\leq p\leq \infty$. For $M\in\N,$ let 
$\phi: \R^M\to \mathcal{L}_p([0,1]^d)$ be an arbitrary function. Suppose there is a continuous function $\mathcal{P}:F_{n,d}^p\to\R^M$ such that $\left\| f-\phi(\mathcal{P} (f)) \right\|_{p}\leq \epsilon$ for all $f\in F_{n,d}^p$. Then $M\gtrsim  \epsilon^{-d/n} $.
\end{theorem}
In~\cite{YAROTSKY2017103} it was observed that when taking $M$ as the number of weights and $\phi$ as a function mapping from the weight space to functions realized by neural networks one directly obtains a lower complexity bound. We note that the increased regularity (expressed in terms of $n$) of the function to be approximated implies a potentially better approximation rate, something which is also apparent in many results to follow. 

\subsection{Upper Complexity Bounds}\label{subsec:shallow_upper_bounds}
Now we turn our attention to results examining the sharpness of the deduced lower bounds by deriving upper bounds. 
In principle, it has been proven in~\cite{Maiorov1999LowerBounds} that there indeed exist sigmoidal, strictly increasing, activation functions $\sigma\in C^\infty(\R)$ for that the bound $N\lesssim \epsilon^{-(d-1)/n}$ of Theorem~\ref{thm:lowShallow} is attained. However, the construction of such an activation function is based on the separability of the space $C([-1,1])$ and hence is not useful in practice. A more relevant upper bound has been given in \cite{Mhaskar:1996}.

\begin{theorem}[\cite{Mhaskar:1996}]\label{thm:mhaskar1996}
Let $n\in \N,~ p\in [1,\infty].$ Moreover, let $\sigma:\R\to\R$ be a function such that $\sigma|_I \in C^\infty(I)$ for some open interval $I\subset \R$ and $\sigma^{(k)}(x_0)\neq 0$ for some $x_0\in I$ and all $k\in \N_0$. Then, for every $f\in F_{n,d}^p$ and every $\epsilon>0,$ there exists a shallow neural network $\Phi_{f,\epsilon}\in \mathcal{N}_{(d,N,1)}$ such that 
\[\left\|f-\act{\Phi_{f,\epsilon}} \right\|_{p} \leq \epsilon,\] and 
$N \lesssim\epsilon^{-d/n}.$
\end{theorem}

We note that this rate is optimal if one assumes \emph{continuous} dependency of the parameters on the approximating function (see~Theorem~\ref{thm:lower_devore}). This is fulfilled in the proof of the upper bounds of Theorem \ref{thm:mhaskar1996}, since it is based on the approximation of Taylor polynomials by realizations of shallow neural networks.
It has been mentioned in \cite[Section 4]{Mhaskar:1996} that this rate can be improved if the function $f$ is analytic. 

The aforementioned lower and upper bounds suffer from the curse of dimensionality. This can for example be avoided if the function class under consideration is assumed to have strong regularity. As an example we state a result given in~\cite{Barron1993,Barron1994,Makozov} where finite Fourier moment conditions are used.

\begin{theorem}[\cite{Barron1993,Barron1994,Makozov}]
Let $\sigma$ be a bounded, measurable, and sigmoidal function. Then, for every \[f\in  \left\{g:\R^d\to \R,~ \int_{\R^d} |\xi|\cdot |\mathcal{F}g(\xi)|~ d\xi < \infty\right\},\] where $\mathcal{F}g$ denotes the Fourier Transform of $g,$ and for every $\epsilon>0,$ there exists a shallow neural network $\Phi_{f,\epsilon}\in \mathcal{N}_{(d,N,1)}$ with 
\[\left\|f-\act{\Phi_{f,\epsilon}} \right\|_{2} \leq \epsilon,\] and $N \lesssim\epsilon^{-2d/(d+1)}.$
\end{theorem}

\noindent Although the dimension appears in the underlying rate, a curse of dimensionality is absent.

Lastly, we present a result where a complexity bound is derived for approximations of a \emph{finite} set of test points. It was shown in \cite{sontag1992feedforward} that if $\sigma:\R\to \R$ is sigmoidal and differentiable in one point $x\in \R$ with non-zero derivative, and if $(x_1,y_1),\dots,$ $(x_{2N+1},y_{2N+1})\in \R\times \R$ for some $N\in \N$ is a set of test points, then, for every $\epsilon>0$, there exists a neural network $\Phi_\epsilon\in \mathcal{N}_{(1,N+1,1)}$ such that 
\[ \sup_{i=1,\dots,2N+1}\left|\act{\Phi_\epsilon}(x_i)-y_i \right| \leq \epsilon.\]

This concludes the part where we examine shallow neural networks. Because of the multitude of existing results we could only discuss a representative selection. For a more comprehensive overview focusing solely on shallow neural networks we refer to \cite{pinkus1999approximation}. 

\section{Universality of Deep Neural Networks}\label{sec:Universal}

So far, the focus was entirely on shallow neural networks. In practice, however, the use of \emph{deep} neural networks, i.e.\ networks with $L>2$ layers, has established itself. 

First attempts to study the expressivity of deep neural networks are based on the \emph{Kolmogorov Superposition Theorem} (see, for instance, \cite{kolmogorov1957representation, KolmogorovThm}). A variant of this theorem (see \cite{sprecher1965structure}) states that every continuous function $f:[0,1]^d\to \R$ can be \emph{exactly} represented as
\begin{align} \label{eq:Kolmogorov}
    f(x_1,\dots, x_d)=\sum_{i=1}^{2d+1}g\left(\sum_{j=1}^dk_j\phi_i(x_j)\right).
\end{align}
Here, $k_j>0$ for $j=1,\dots,d$, such that $\sum_{j=1}^d k_j\le 1$, $\phi_i:[0,1]\rightarrow [0,1]$, $i=1,\dots, 2d+1$, strictly increasing and $g\in C([0,1])$ are functions depending on $f$.
In~\cite{hecht1987kolmogorov}, Equation~\eqref{eq:Kolmogorov} has been interpreted in the context of (a general version of) a neural network. It yields that every continuous function can be exactly represented by a 3-layer neural network with width $2d^2+d$  and different activation functions in each neuron depending on the function to be approximated.
However, quoting~\cite{hecht1987kolmogorov}, the ``direct usefulness of this result is doubtful'', since it is not known how to construct the functions $g,\phi_1,\dots,\phi_{2d+1},$ which play the roles of the activation functions. Furthermore, in \cite{Poggio1989Kolmogorov} it has been pointed out that the dependence of the activation functions on $f$ makes this approach hardly usable in practice. We refer to~\cite{Poggio1989Kolmogorov} for a more detailed discussion about the suitability of the Kolmogorov Superposition Theorem in this context.
 
Subsequent contributions focused on more practical neural network architectures, where the activation functions are fixed a priori and only the parameters of the affine maps are adjusted.
Under these assumptions, however, the capability of representing every function in $C([0,1]^d)$ in an exact way is lost. Consequently, the expressivity of neural networks with a fixed activation function has been studied in terms of their approximative power for specific function classes $\Ccal$. We note that the universality results for shallow neural networks from the last section in general also hold for \emph{deep} neural networks with a fixed number of layers (see \cite{FUNAHASHI1989,Hornik1989universalApprox}). In~\cite[Corollary 2.7]{Hornik1989universalApprox} universality for shallow networks in $C(K)$ is transferred to the multi-layer case via~\cite[Lemma A.6]{Hornik1989universalApprox}. It states that if $F,G\subset C(\R)$ are such that $F|_K,G|_K$ are dense sets for every compact subset $K\subset \R$, then also $\{f\circ g:f\in F,g\in G\}$ is dense in $C(K).$

In contrast to the lower bounds in~Theorem \ref{thm:lowShallow}, for the case of three layer neural networks it is possible to show the existence of a pathological (i.e.\ in practice unusable) activation function $\sigma$ such that the set $\RNcal_{(d,2d+1, 4d+3,1),\sigma}$ is dense in $C([-1,1]^d)$ (see~\cite{Maiorov1999LowerBounds}). As in the case of Equation~\eqref{eq:Kolmogorov}, the remarkable independence of the complexity on the approximation error and its mere linear dependence on $d$ (implying that the curse of dimensionality can be circumvented) is due to the choice of the activation function. However, for practically used activation functions such as the ReLU, parametric ReLU, exponential linear unit, softsign and tanh, universality in $C(K)$ does not hold (\cite[Remark 2.3]{TopNet}). 

The dual problem of considering neural networks with \emph{fixed depth} and \emph{unbounded width}, is to explore expressivity of neural networks with \emph{fixed width} and \emph{unrestricted depth}. 
In~\cite{hanin2017approximating,HaninFiniteWidth} it is shown that the set of ReLU neural networks with width $\geq d+n$ and unrestricted depth is an universal approximator for the function class $\Ccal=C([0,1]^d,\R^n)$. In the case $n=1,$ the lower bound on the width is sharp. For $\Ccal= \mathcal{L}_1(\R^d),$ the paper \cite{WidthNIPS} establishes universality of deep ReLU neural networks with width $\geq d+4.$ The necessary width for ReLU neural networks to yield universal approximators is bounded from below by~$d.$ 

\section{Approximation of Classes of Smooth Functions} \label{sec:ApproxSmooth}

In this section, we examine approximation rates of \emph{deep} neural networks for functions characterized by smoothness properties. This section can be seen as a counterpart of Subsection~\ref{subsec:shallow_lower_bounds} and~\ref{subsec:shallow_upper_bounds} with three major differences: We now focus on deep neural networks (instead of shallow ones); most of these results were shown after the rise of deep learning in 2012; currently used activation functions (like e.g. the ReLU\footnote{The first activation function that was used was the \emph{threshold function} $\sigma = \mathbf{1}_{[0,\infty)}$ (see \cite{MP43}). This was biologically motivated and constitutes a mathematical interpretation of the fact that "a neuron fires if the incoming signal is strong enough". Since this function is not differentiable everywhere and its derivative is zero almost everywhere, smoothened versions (sigmoidal functions) have been employed to allow for the usage of the backpropagation algorithm. These functions, however, are subject to the \emph{vanishing gradient problem}, which is not as common for the ReLU (see \cite{VanGradient}). Another advantage of the ReLU is that it is easy to compute and promotes sparsity in data representation (see~\cite{ReLUSparsity}). 
}) are analyzed. 

A groundlaying result has been given in~\cite{YAROTSKY2017103}. There it has been shown that for each $\epsilon>0,$ and for each function $f$ in $F^\infty_{n,d}$, there exists a ReLU neural network with $L\lesssim\log_2(1/\eps)$ layers, as well as $M,N \lesssim\eps^{-d/n}\log_2(1/\eps)$ weights and neurons capable of approximating $f$ with $\mathcal{L}_\infty$-approximation error $\epsilon$. A generalization of this result for functions from $F_{n,d}^p$ with error measured in the $W^{r,p}$-norm\footnote{Here, $W^{r,p}$ for $r\in (0,1)$ denotes the Sobolev-Slobodeckij spaces as considered in \cite[Definition 3.3]{guhring2019error}} for $0\leq r\leq 1$ was obtained in~\cite{guhring2019error}. It shows that the error can also be measured in norms that include the distance of the derivative and that there is a trade-off between the regularity $r$ used in the approximating norm and the complexity of the network. The following theorem summarizes the findings of \cite{YAROTSKY2017103}\footnote{The results are presented in a slightly modified version. The neural networks considered in~\cite{YAROTSKY2017103} are allowed to have skip connections possibly linking a layer to all its successors. The rates obtained are equal, only the square power of the logarithm needs to be removed.} and \cite{guhring2019error}.
\begin{theorem}[\cite{YAROTSKY2017103, guhring2019error}]\label{thm:yarotsky}
For every $f\in F_{n,d}^p$ there exists a neural network $\Phi_{f,\eps}$ whose realization is capable of approximating $f$ with error $\eps$ in $W^{r,p}$-norm $(0\leq r\leq 1)$, with
\begin{itemize}
    \item[(i)] $M(\Phi_{f,\eps}) \lesssim\epsilon^{-d/(n-r)}\cdot\log^2_2\left(\eps^{-n/(n-r)}\right),$ and
    \item[(ii)] $L(\Phi_{f,\eps}) \lesssim \log_2\left(\eps^{-n/(n-r)}\right)$.
\end{itemize}
\end{theorem}
\begin{remark}\label{rem:Square}
The depth scales logarithmically in the approximation accuracy. This is due to the ReLU, which renders the approximation of the map $x\mapsto x^2$ difficult. For activation functions $\sigma$ with $\sigma(x_0),$ $\sigma^{'}(x_0),$ $\sigma^{''}(x_0)\neq 0,$ for some $x_0\in \R$ it is possible to approximate $x\mapsto x^2$ by a neural network with a fixed number of weights and $L=2$ (see in~\cite{PowerOfDepth}).
\end{remark}

A common tool for deducing upper complexity bounds is based on the approximation of suitable representation systems. Smooth functions, for instance, can be locally described by (Taylor) polynomials, i.e. 
\[f\approx \sum_{\|\alpha\|_1\leq k}^n c_k (\cdot- {x}_0)^\alpha \quad \text{ locally.}\]
Approximating monomials by neural networks then yields an approximation of $f$.\footnote{For shallow neural networks this ansatz has e.g. been used in \cite{Mhaskar:1996}.} The proof strategy employed by Yarotsky is also based on this idea. It was picked up by several follow-up works (\cite{guhring2019error, PetV2018OptApproxReLU}) and we describe it here in more detail.

\emph{Sketch of Proof:} The core of the proof is an approximation of the square function $x\mapsto x^2$ by a piecewise linear interpolation that can be expressed by ReLU neural networks (see Figure \ref{fig:YarotskyUpper} for a visualization). First, we define $g\colon [0,1] \to [0,1]$, by $g(x) \coloneqq \min \{2x, 2-2x\}$. Notice that the hat function $g$ is representable by a ReLU neural network. 
Multiple compositions of $g$ with itself result in saw-tooth functions (see Figure~\ref{fig:YarotskyUpper}). We set, for $m \in \N$, $g_1 \coloneqq g$ and $g_{m+1} \coloneqq g \circ g_m$. 
It was demonstrated in \cite{YAROTSKY2017103} that
    \begin{align*}
        x^2 = \lim_{n \to \infty} f_n(x)  \coloneqq \lim_{n \to \infty} x - \sum_{m = 1}^n \frac{g_m(x)}{2^{2m}}, \quad \text{ for all } x \in [0,1].
    \end{align*} 
    
    \begin{figure}
    \centering
    \begin{subfigure}{.4\textwidth}
    \if\b0
\begin{tikzpicture}[scale=0.7, transform shape]

\definecolor{color0}{rgb}{0.12156862745098,0.466666666666667,0.705882352941177}
\definecolor{color1}{rgb}{1,0.498039215686275,0.0549019607843137}
\definecolor{color2}{rgb}{0.172549019607843,0.627450980392157,0.172549019607843}

\begin{axis}[
legend cell align={left},
legend style={at={(0.5,0.09)}, anchor=south, draw=white!80.0!black},
tick align=outside,
tick pos=both,
x grid style={white!69.01960784313725!black},
xmin=0, xmax=1,
xtick style={color=black},
y grid style={white!69.01960784313725!black},
ymin=0, ymax=1,
ytick style={color=black}
]
\addplot [semithick, color0]
table {%
0 0
0.5 1
1 0
};
\addlegendentry{g}
\addplot [semithick, color1]
table {%
0 0
0.25 1
0.5 0
0.75 1
1 0
};
\addlegendentry{g2}
\addplot [semithick, color2]
table {%
0 0
0.125 1
0.25 0
0.375 1
0.5 0
0.625 1
0.75 0
0.875 1
1 0
};
\addlegendentry{g3}
\end{axis}

\end{tikzpicture}
\fi

\if\b1
\begin{tikzpicture}[scale=0.7, transform shape]

\definecolor{color0}{rgb}{0.12156862745098,0.466666666666667,0.705882352941177}
\definecolor{color1}{rgb}{1,0.498039215686275,0.0549019607843137}
\definecolor{color2}{rgb}{0.172549019607843,0.627450980392157,0.172549019607843}

\begin{axis}[
legend cell align={left},
legend style={at={(0.5,0.09)}, anchor=south, draw=white!80.0!black},
tick align=outside,
tick pos=both,
x grid style={white!69.01960784313725!black},
xmin=0, xmax=1,
xtick style={color=black},
xtick={0,0.2,0.4,0.6,0.8,1},
xticklabels={ ,,,,,},
y grid style={white!69.01960784313725!black},
ymin=0, ymax=1,
ytick style={color=black},
ytick={0,0.2,0.4,0.6,0.8,1},
yticklabels={ ,,,,,}
]
\addplot [semithick, color0]
table {%
0 0
0.5 1
1 0
};
\addlegendentry{g}
\addplot [semithick, color1]
table {%
0 0
0.25 1
0.5 0
0.75 1
1 0
};
\addlegendentry{g2}
\addplot [semithick, color2]
table {%
0 0
0.125 1
0.25 0
0.375 1
0.5 0
0.625 1
0.75 0
0.875 1
1 0
};
\addlegendentry{g3}
\end{axis}

\end{tikzpicture}
\fi

\if\b2
\begin{tikzpicture}[scale=0.7, transform shape]

\definecolor{color0}{rgb}{0.12156862745098,0.466666666666667,0.705882352941177}
\definecolor{color1}{rgb}{1,0.498039215686275,0.0549019607843137}
\definecolor{color2}{rgb}{0.172549019607843,0.627450980392157,0.172549019607843}

\begin{axis}[
legend cell align={left},
legend style={at={(0.5,0.09)}, anchor=south, draw=white!80.0!black,row sep=3,font=\huge},
tick pos=both,
xmin=0, xmax=1,
ymin=0, ymax=1,
xtick style={color=white},
ytick style={color=white},
yticklabels={ ,,,,,},
xticklabels={ ,,,,,}
]
\addplot [semithick, color0]
table {%
0 0
0.5 1
1 0
};
\addlegendentry{$g$}
\addplot [semithick, color1]
table {%
0 0
0.25 1
0.5 0
0.75 1
1 0
};
\addlegendentry{$g_2$}
\addplot [semithick, color2]
table {%
0 0
0.125 1
0.25 0
0.375 1
0.5 0
0.625 1
0.75 0
0.875 1
1 0
};
\addlegendentry{$g_3$}
\end{axis}

\end{tikzpicture}
\fi
    \end{subfigure}
    $\qquad\qquad$ 
    \begin{subfigure}{.4\textwidth}
    \if\b0
\begin{tikzpicture}[scale=0.7, transform shape]

\definecolor{color0}{rgb}{0.12156862745098,0.466666666666667,0.705882352941177}
\definecolor{color1}{rgb}{1,0.498039215686275,0.0549019607843137}
\definecolor{color2}{rgb}{0.172549019607843,0.627450980392157,0.172549019607843}
\definecolor{color3}{rgb}{0.83921568627451,0.152941176470588,0.156862745098039}

\begin{axis}[
legend cell align={left},
legend style={at={(0.03,0.97)}, anchor=north west, draw=white!80.0!black},
tick align=outside,
tick pos=both,
x grid style={white!69.01960784313725!black},
xmin=0, xmax=1,
xtick style={color=black},
y grid style={white!69.01960784313725!black},
ymin=0, ymax=1,
ytick style={color=black}
]
\addplot [semithick, color0, dashed]
table {%
0 0
0.1 0.01
0.2 0.04
0.3 0.09
0.4 0.16
0.5 0.25
0.6 0.36
0.7 0.49
0.8 0.64
0.9 0.81
1 1
};
\addlegendentry{$x\mapsto x^2$}
\addplot [semithick, color1]
table {%
0 0
0.1 0.1
0.2 0.2
0.3 0.3
0.4 0.4
0.5 0.5
0.6 0.6
0.7 0.7
0.8 0.8
0.9 0.9
1 1
};
\addlegendentry{f0}
\addplot [semithick, color2]
table {%
0 0
0.1 0.05
0.2 0.1
0.3 0.15
0.4 0.2
0.5 0.25
0.6 0.4
0.7 0.55
0.8 0.7
0.9 0.85
1 1
};
\addlegendentry{f1}
\addplot [semithick, color3]
table {%
0 0
0.1 0.025
0.2 0.05
0.3 0.1
0.4 0.175
0.5 0.25
0.6 0.375
0.7 0.5
0.8 0.65
0.9 0.825
1 1
};
\addlegendentry{f2}
\end{axis}

\end{tikzpicture}
\fi
\if\b1
\begin{tikzpicture}[scale=0.7, transform shape]

\definecolor{color0}{rgb}{0.12156862745098,0.466666666666667,0.705882352941177}
\definecolor{color1}{rgb}{1,0.498039215686275,0.0549019607843137}
\definecolor{color2}{rgb}{0.172549019607843,0.627450980392157,0.172549019607843}
\definecolor{color3}{rgb}{0.83921568627451,0.152941176470588,0.156862745098039}

\begin{axis}[
legend cell align={left},
legend style={at={(0.03,0.97)}, anchor=north west, draw=white!80.0!black},
tick align=outside,
tick pos=both,
x grid style={white!69.01960784313725!black},
xmin=0, xmax=1,
xtick style={color=black},
xtick={0,0.2,0.4,0.6,0.8,1},
xticklabels={ ,,,,,},
y grid style={white!69.01960784313725!black},
ymin=0, ymax=1,
ytick style={color=black},
ytick={0,0.2,0.4,0.6,0.8,1},
yticklabels={ ,,,,,}
]
\addplot [semithick, color0, dashed]
table {%
0 0
0.1 0.01
0.2 0.04
0.3 0.09
0.4 0.16
0.5 0.25
0.6 0.36
0.7 0.49
0.8 0.64
0.9 0.81
1 1
};
\addlegendentry{$x\mapsto x^2$}
\addplot [semithick, color1]
table {%
0 0
0.1 0.1
0.2 0.2
0.3 0.3
0.4 0.4
0.5 0.5
0.6 0.6
0.7 0.7
0.8 0.8
0.9 0.9
1 1
};
\addlegendentry{f0}
\addplot [semithick, color2]
table {%
0 0
0.1 0.05
0.2 0.1
0.3 0.15
0.4 0.2
0.5 0.25
0.6 0.4
0.7 0.55
0.8 0.7
0.9 0.85
1 1
};
\addlegendentry{f1}
\addplot [semithick, color3]
table {%
0 0
0.1 0.025
0.2 0.05
0.3 0.1
0.4 0.175
0.5 0.25
0.6 0.375
0.7 0.5
0.8 0.65
0.9 0.825
1 1
};
\addlegendentry{f2}
\end{axis}

\end{tikzpicture}
\fi

\if\b2
\begin{tikzpicture}[scale=0.7, transform shape]

\definecolor{color0}{rgb}{0.12156862745098,0.466666666666667,0.705882352941177}
\definecolor{color1}{rgb}{1,0.498039215686275,0.0549019607843137}
\definecolor{color2}{rgb}{0.172549019607843,0.627450980392157,0.172549019607843}
\definecolor{color3}{rgb}{0.83921568627451,0.152941176470588,0.156862745098039}

\begin{axis}[
legend cell align={left},
legend style={at={(0.03,0.97)}, anchor=north west, draw=white!80.0!black,row sep=1,font=\huge},
tick pos=both,
xmin=0, xmax=1,
ymin=0, ymax=1,
xtick style={color=white},
xtick={0,0.2,0.4,0.6,0.8,1},
xticklabels={ ,,,,,},
y grid style={white!69.01960784313725!black},
ytick style={color=white},
ytick={0,0.2,0.4,0.6,0.8,1},
yticklabels={ ,,,,,}
]
\addplot [semithick, color0, dashed]
table {%
0 0
0.1 0.01
0.2 0.04
0.3 0.09
0.4 0.16
0.5 0.25
0.6 0.36
0.7 0.49
0.8 0.64
0.9 0.81
1 1
};
\addlegendentry{$x\mapsto x^2$}
\addplot [semithick, color1]
table {%
0 0
0.1 0.1
0.2 0.2
0.3 0.3
0.4 0.4
0.5 0.5
0.6 0.6
0.7 0.7
0.8 0.8
0.9 0.9
1 1
};
\addlegendentry{$f_0$}
\addplot [semithick, color2]
table {%
0 0
0.1 0.05
0.2 0.1
0.3 0.15
0.4 0.2
0.5 0.25
0.6 0.4
0.7 0.55
0.8 0.7
0.9 0.85
1 1
};
\addlegendentry{$f_1$}
\addplot [semithick, color3]
table {%
0 0
0.1 0.025
0.2 0.05
0.3 0.1
0.4 0.175
0.5 0.25
0.6 0.375
0.7 0.5
0.8 0.65
0.9 0.825
1 1
};
\addlegendentry{$f_2$}
\end{axis}

\end{tikzpicture}
\fi
    \end{subfigure}
    \caption{Visualization of the approximation of the square function $x\mapsto x^2$ by ReLU realizations of neural networks as in \cite{YAROTSKY2017103}.}
    \label{fig:YarotskyUpper}
\end{figure}
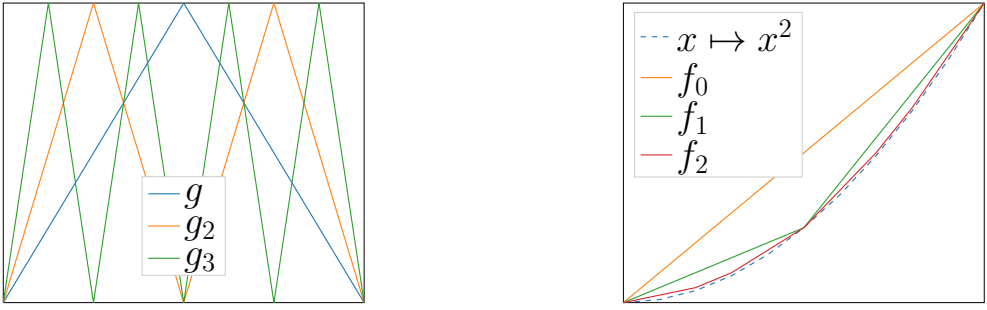
\noindent Hence, there exist neural networks $\Phi_{x^2,\epsilon}$ the ReLU realizations of which approximate $x\mapsto x^2$ uniformly on $[0,1]$ up to an error of $\epsilon$. It can be shown that $M(\Phi_{x^2,\epsilon})$, $L(\Phi_{x^2,\epsilon}), N(\Phi_{x^2,\epsilon})\lesssim \log_2(1/\epsilon)$.
From this a neural network $\Phi_{\mathrm{mult}}$ is constructed via the polarization identity 
\begin{equation*}
	xy=\frac{1}{2}((x+y)^2-x^2-y^2)\qquad\text{for }x,y\in\R,	
\end{equation*}
the ReLU realizations of which locally approximates the multiplication map $(x,y)\mapsto xy$. It is now straight-forward to approximate arbitrary polynomials by realizations of neural networks with $L\lesssim \log_2(1/\epsilon)$ layers and $M\lesssim \mathrm{polylog}(1/\epsilon)$ weights.\footnote{In~\cite{TelgarskyRational}, approximation rates for polynomials have been extended to rational functions. It has been shown, that one can locally uniformly approximate rational functions $f:[0,1]^d\to \R$ up to an error $\epsilon>0$ by ReLU realizations of neural networks $\Phi_{f,\epsilon}$ of size $M\left(\Phi_{f,\epsilon}\right)\lesssim\mathrm{poly}(d)\cdot \mathrm{polylog}(1/\epsilon).$ }
In order to globally approximate $f,$ a partition of unity is constructed with ReLU neural networks and combined with the approximate polynomials. Since the domain needs to be partitioned into roughly $\eps^{-d/n}$ patches, the curse of dimensionality occurs in these bounds.

\hfill $\square$

From Theorem~\ref{thm:lower_devore} we can deduce that under the hypothesis of a continuous dependency of weights and biases on the function $f$ to be approximated, the upper bound shown in Theorem~\ref{thm:yarotsky} for the case $r=0$ is tight (up to a log factor).
If the assumption of continuous weight dependency is dropped, Yarotsky derived a bound for the case $r=0, p=\infty$ and in~\cite{guhring2019error} the case $r=1,p=\infty$ is covered. Both cases are combined in the next theorem. For its exposition we need to introduce additional notation. Analogously to \cite{YAROTSKY2017103}, we denote by $\Arch$ a neural network with unspecified non-zero weights and call it neural network architecture. We say that \emph{the architecture $\mathcal{A}$ is capable of approximating a function $f$ with error $\eps$ and activation function $\sigma$, if this can be achieved by the $\sigma$-realization of some weight assignment.}

\begin{theorem}[\cite{YAROTSKY2017103, guhring2019error}]\label{thm:YAROTSKYLower}
Let $\epsilon>0$ and $r\in\{0,1\}$. If $\Arch_\eps$ is a neural network architecture that is capable of approximating every function from $F^{\infty}_{n,d}$ with error $\eps$ in $W^{r,\infty}$ norm and with activation function ReLU, then 
\[
M(\Arch_\eps)\gtrsim \eps^{\nicefrac{-d}{2(n-r)}}.
\]
\end{theorem}

\begin{remark}
The gap between the upper bound $M_\eps\lesssim\epsilon^{-d/n}$ of Theorem~\ref{thm:yarotsky} and the lower bound $M_\eps \gtrsim\epsilon^{-d/{2n}}$ for $r=0$ is discussed in~\cite{YarotskyPhase2019,yarotsky18a}. It is an instance of the benefit of (in this instance very) deep over shallow neural networks. It is shown that for every $\zeta\in \left[d/(2n), d/n \right)$ there exist neural networks with $M_\eps\lesssim\epsilon^{-\zeta}$ non-zero weights and $L_\eps\lesssim\epsilon^{-d/(r(\zeta-1))}$ layers that uniformly $\eps$-approximate functions in $F^{\infty}_{n,d}$.
\end{remark}

\emph{Sketch of Proof:} We only proof the statement for $r=0$. The case $r=1$ is proven similarly. The proof provides a general way of showing lower complexity bounds based on the \emph{Vapnik-Chervonenkis dimension (VCdim)} \cite{VCdim}. The $\vcdim$ measures the expressiveness of a set of binary valued functions $H$ defined on some set $A,$ and is defined by
\[
	\vcdim(H)\coloneqq\sup\left\{m\in\N\,: \begin{array}{l}
		\text{there exist }x_1,\ldots,x_m\in A\text{ such that}\\[0.3em]
		\text{for every }y\in\{0,1\}^m\text{ there is a function}\\[0.3em]
		h\in H \text{ with }h(x_i)=y_i\text{ for }i=1,\ldots,m
	\end{array}
		\right\}.
	\]
We define the set of thresholded realizations of neural networks 
\[
H\coloneqq\left\{\mathbf{1}_{(-\infty,a]}\circ \act{\Phi_\theta}:\theta\in\R^{M(\Phi)}\right\},
\] for some (carefully chosen) $a\in\R$ and derive the chain of inequalities
	\begin{equation}\label{eq:vcdim_lower_bounds}
		c \cdot \eps^{-d/n}\leq \vcdim(H)\leq C\cdot M(\Phi_\eps)^2.
		\end{equation}
		
The upper bound on $\vcdim(H)$ in~Equation~\eqref{eq:vcdim_lower_bounds} has been shown in \cite[Theorem~8.7]{anthony2009neural}.

To establish the lower bound, set $N\coloneqq\floor*{\eps^{-1/n}}$ and let \[
x_1,\ldots,x_{N^d}\in [0,1]^d\quad \text{such that}\quad \pabs{x_m-x_n}\geq 1/N\] for all $m,n=1,\ldots,N^d$ with $m\neq n$. For arbitrary $y=(y_1,\ldots,y_{N^d})\in\{0,1\}^{N^d},$ Yarotsky constructs a function $f_y\in F^{\infty}_{n,d}$ with $f_y(x_m)=y_m\cdot N^{-n}$ for $m=1,\ldots,N^d$. Let now $\Phi_{f_y}$ be a neural network such that $\act{\Phi_{f_y}}$ $\eps$-approximates $f_y$, then we have for thresholded neural network realization $\mathbf{1}_{(-\infty,a]}\circ\act{\Phi_{f_y}}$ that $\act{\Phi_{f_y}}(x_m)=y_m$ (see Figure \ref{fig:YarotskyLower}).
\begin{figure}[ht!] 
\centering
    \includegraphics[width=.7\textwidth]{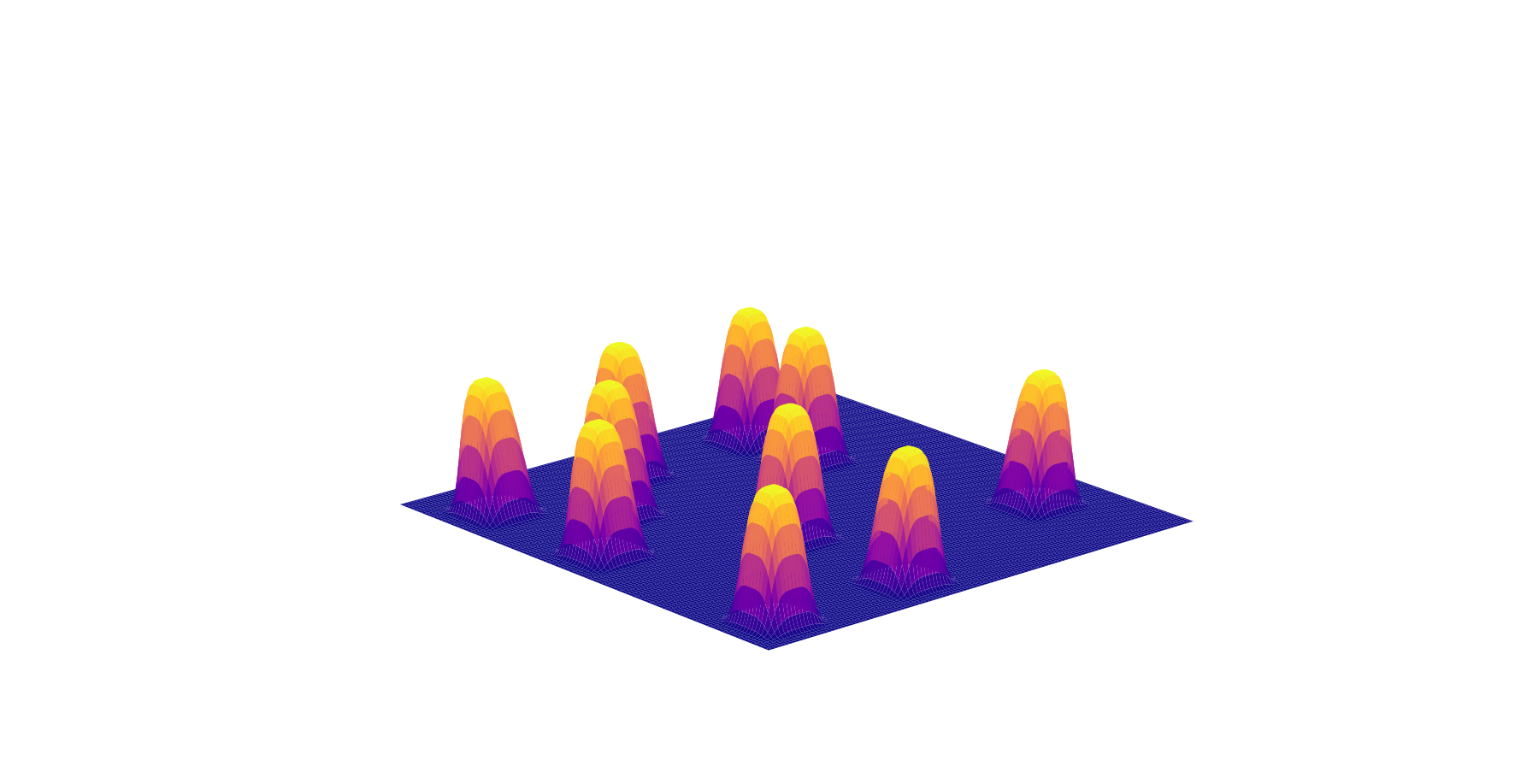}
     \caption{The function $f_y$ in $d=2$ dimensions.}\label{fig:YarotskyLower}
\end{figure}
Using the definition of $\vcdim$ it is easy to see that
\[
c\cdot\eps^{-d/n}\leq N^d\leq \vcdim(H),
\]
which is the lower bound in~Equation~\eqref{eq:vcdim_lower_bounds}. The theorem now easily follows from Equation~\eqref{eq:vcdim_lower_bounds}. 

\hfill $\square$

Approximations in $\mathcal{L}_p$ norm for $\beta$-H\"older-continuous functions have been considered in~\cite{schmidt2017nonparametric,PetV2018OptApproxReLU}. In contrast to~\cite{YAROTSKY2017103,guhring2019error} the depth of the involved networks remains fixed and does not depend on the approximation accuracy. Additionally in~\cite{PetV2018OptApproxReLU}, the weights are required to be encodable.\footnote{I.e., the weights are representable by no more than $\sim\log_2(1/\epsilon)$ bits.} We summarize their findings in the following theorem:

\begin{theorem}[\cite{PetV2018OptApproxReLU}]\label{thm:PETVOIGTSMOOTH}
Let $\beta=(n,\zeta)$ for $n\in \N_0,~\zeta\in (0,1]$ and $p\in(0,\infty).$ Then, for every $\epsilon>0$ and every $f\in C^\beta([-1/2,1/2]^d)$ with $\|f\|_{C^\beta}\leq 1,$ there exist ReLU neural networks $\Phi_{f,\epsilon}$ with encodable weights, $L(\Phi_{f,\epsilon})\lesssim\log_2((n+\zeta))\cdot (n+\zeta)/d$ layers and $M(\Phi_{f,\epsilon})\lesssim\epsilon^{-d/(n+\zeta)}$ non-zero weights such that 
    \[\left\|f-\act{\Phi_{f,\epsilon}} \right\|_{p}\leq \epsilon.\]
\end{theorem} 

There also exist results based on the approximation of B-splines (\cite{Mhaskar1993}) or finite elements (\cite{he2018relu,OPS19_811}). It is shown there that neural networks perform as well as the underlying approximation procedure.

Finally, instead of examining the approximation rates of deep neural networks for specific function classes, one could also ask the following question: \emph{ Which properties does the set of \emph{all} functions that can be approximated by deep neural networks at a given rate fulfill?} This question has been discussed extensively in~\cite{ApproxSpaces2019}. Among other results, it has been shown that, under certain assumptions on the neural network architecture, these sets of functions can be embedded into classical function spaces such as Besov spaces of a certain degree of smoothness.

\section{Approximation of Piecewise Smooth Functions} \label{sec:DeepPiecewiseRates}

When modeling real-world phenomena one often only assumes \emph{piecewise} smoothness. A prominent example are \emph{cartoon-like functions} (see \cite{DCartoonLikeImages2001})
        \[
	\mathcal{E}^n\big([0,1]^d\big)\coloneqq\left\{ f_1 + \mathbf{1}_B f_2:\begin{array}{l}
		f_1,f_2 \in C^n\left([0,1]^d\right),\, B \subset (0,1)^d,\, \partial B  \in C^n \\
		 \text{and } \|g\|_{C^n}\leq 1  \text{ for } g = f_1,f_2,\partial B
	\end{array}
		\right\},
	\]
which are commonly used as a mathematical model for images. Figure~\ref{fig:cartoon} provides an illustration for the case $d=2$.
\begin{figure}[ht!] 
\centering
    \includegraphics[width=0.25\textwidth]{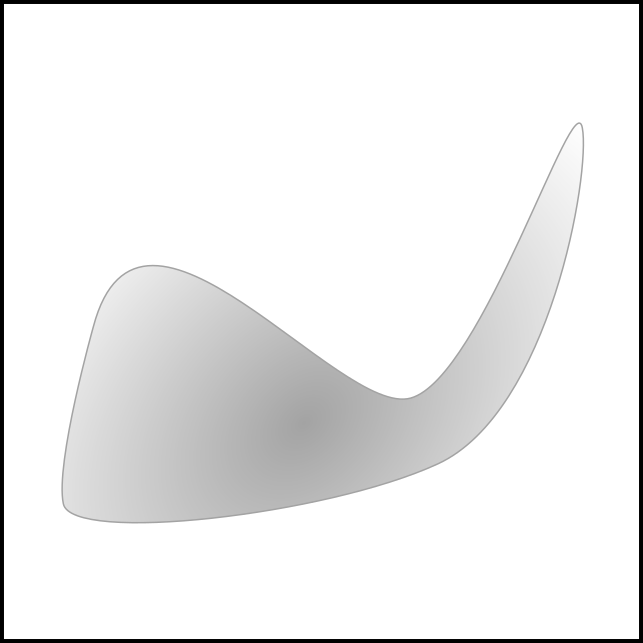}
    \caption{Illustration of a cartoon-like function on $[0,1]^2$.}\label{fig:cartoon}
\end{figure}
First expressivity results in this direction have been deduced in~\cite{boelcskeiNeural} for neural networks with weights of restricted complexity\footnote{Note, that computers can also only store weights of restricted complexity.}. To present this result, we first need to introduce some notions from information theory.

The \emph{minimax code-length} describes the necessary length of bitstrings of an encoded representation of functions from a function class $\mathcal{C}\subset \mathcal{L}_2(K)$ such that it can be decoded with an error smaller then $\eps>0$. The precise definition is given as follows:
\begin{definition}[see \cite{boelcskeiNeural} and the references therein]\label{def:EncDec}
Let $K \subset \R^d$ be measurable, and let $\mathcal{C}\subset \mathcal{L}_2(K)$ be compact. For each $\ell\in \N$, we denote by
\[
    \mathfrak{E}^\ell\coloneqq \left\{E: \mathcal{C} \to \{0,1\}^{\ell}\right\},
\]
the set of \emph{binary encoders mapping elements of $\mathcal{C}$ to bit-strings of length $\ell$}, and we let
\[
    \mathfrak{D}^\ell\coloneqq \left\{D:\{0,1\}^{\ell} \to  \mathcal{L}_2(K)\right\},
\]
be the set of \emph{binary decoders mapping bit-strings of length $\ell$ to elements of $\mathcal{L}_2(K)$}.

An encoder-decoder pair $(E^\ell, D^\ell) \in \mathfrak{E}^\ell \times \mathfrak{D}^\ell$
is said to {\em achieve distortion $\epsilon >0$ over the function class $\mathcal{C}$}, if
\begin{align*}
    \sup_{f\in \mathcal{C}} \left\|D^\ell (E^\ell (f)) - f \right\|_{2} \leq \epsilon.
\end{align*}
Finally, for $\epsilon >0$ the \emph{minimax code length} $L(\epsilon, \mathcal{C})$ is
\[
    L(\epsilon, \mathcal{C})
    \coloneqq \min\left\{\begin{array}{l}
              \ell\in \N
              \,:\,
              \exists \left(E^\ell, D^\ell\right) \in  \mathfrak{E}^\ell \times \mathfrak{D}^\ell: \\
                  \sup_{f\in \mathcal{C}} \left\|D^\ell (E^\ell (f)) - f \right\|_{2}
                  \leq \epsilon
                 \end{array}
           \right\},
\]
with the interpretation $L(\epsilon, \mathcal{C}) = \infty$ if
$\sup_{f\in \mathcal{C}} \|D^\ell (E^\ell (f)) - f \|_{2} > \epsilon$ for all 
$(E^\ell, D^\ell) \in \mathfrak{E^\ell} \times \mathfrak{D}^\ell$ and arbitrary $\ell \in \N$.
\end{definition}
We are particularly interested in the asymptotic behavior of $L(\epsilon, \mathcal{C})$, which can be quantified by the \textit{optimal exponent}.

\begin{definition}\label{def:optexp}
Let $K \subset \R^d$ and $\mathcal{C}\subset \mathcal{L}_2(K)$. Then, the \emph{optimal exponent} $\gamma^*(\mathcal{C})$ is defined by
\[
\gamma^*(\mathcal{C}): = \inf \left\{\gamma \in \R: L(\epsilon, \mathcal{C})  \lesssim\epsilon ^{-\gamma}, \text{ as } \epsilon \searrow 0 \right\}.
\]
\end{definition}
The optimal exponent $\gamma^*(\mathcal{C})$ describes how fast $L(\epsilon, \mathcal{C})$ tends to infinity as $\epsilon$ decreases. For function classes $\mathcal{C}_1$ and $\mathcal{C}_2$, the notion $\gamma^*(\mathcal{C}_1) < \gamma^*(\mathcal{C}_2)$ indicates that asymptotically, i.e., for $\epsilon\searrow 0$, the necessary length of the encoding bit string for $\mathcal{C}_2$ is larger than that for $\mathcal{C}_1$. In other words, a smaller exponent indicates a smaller description complexity.
\begin{example}\label{example:optimal_coefficient}
For many function classes the optimal exponent is well-known (see \cite{boelcskeiNeural} and the references therein). Let $ n\in \N$, $1\leq p,q \leq \infty$, then 
\begin{itemize}
    \item[(i)] $\gamma^*\left(\left\{f \in C^{n}([0,1]^d): \|f\|_{C^n} \leq 1 \right\}\right) = d/n$, 
  
    \item[(ii)] if $n\in \{1,2\},$ then $\gamma^*(\mathcal{E}^n([0,1]^d)) = 2(d-1)/n$.
\end{itemize}
\end{example}

The next theorem connects the description complexity of a function class with the necessary complexity of neural network approximations with encodable weights. It shows that, at best, one can hope for an asymptotic growth governed by the optimal exponent.
\begin{theorem}[\cite{boelcskeiNeural}]\label{thm:optimalitynoquant}
Let $K\subset \mathbb{R}^d$, $\sigma :\R \to \R$, $c>0$, and $\mathcal{C} \subset \mathcal{L}_2(K)$. Let $\epsilon\in(0,\nicefrac{1}{2})$ and $M_\eps\in\N$. If for every $f\in \mathcal{C}$ there exists a neural network $\Phi_{\eps,f}$ with weights encodable with $\ceil{c\log_2(\nicefrac{1}{\eps})}$ bits and $\norm{f-\act{\Phi_{\eps,f}}}_{\mathcal{L}_2}\leq \eps$ and $M(\Phi_{\eps,f})\leq M_\eps$, then 
\[
M_\eps\gtrsim \eps^{-\gamma},
\]
for all $\gamma<\gamma^\ast(\mathcal{C})$.
\end{theorem}

We are now interested in the deduction of optimal upper bounds. We have seen already in many instances, that one of the main ideas behind establishing approximation rates for neural networks is to demonstrate how other function systems, often polynomials, can be emulated by them. 
In \cite{shaham2018provable}, a similar approach was followed by demonstrating that neural networks can reproduce wavelet-like functions (instead of polynomials) and thereby also sums of wavelets. This observation allows to transfer \textit{$M$-term approximation rates} with wavelets to $M$-weight approximation rates with neural networks. In~\cite{boelcskeiNeural} this route is taken for general affine systems. An \emph{affine system} is constructed by applying affine linear transformations to a \emph{generating function}. We will not give the precise definition of an affine system here (see e.g.\ \cite{boelcskeiNeural}) but intend to build some intuition by considering shearlets in $\R^2$ as an example.

Shearlet systems (\cite{ShearletsBook}) are representation systems used mainly in signal and image processing. Similar to the Fourier transform, which expands a function in its frequencies, a shearlet decomposition allows an expansion associated to different location, direction and resolution levels. 
To construct a shearlet system $\mathcal{SH}$, one needs a \emph{parabolic scaling} operation defined by the matrix 
\[
	A_j\coloneqq\bmat{c c}{
				2^j & 0\\[1em]
				0 & 2^{j/2}
			},
\]
a \emph{shearing} operation defined by
\[
	S_k\coloneqq\bmat{c c}{
				1 & k\\[1em]
				0 & 1
			},
\] together with the translation operation. These operations are applied to a generating function $\psi\in \mathcal{L}_2(\R^2)$ (satisfying some technical conditions) to obtain a \emph{shearlet system}
\[
\mathcal{SH}\coloneqq\big\{2^{\frac{3j}{4}}\psi(S_k A_j (\cdot)-n):j\in\Z, k\in\Z, n\in\Z^2\big\}.
\]
Shearlet systems are particularly well suited for the class of cartoon-like functions. To make this statement rigorous, we first need the following definition:

\begin{definition}
For a normed space $V$ and a system $(\phi_i)_{i\in I} \subset V$ we define the \textit{error of best $M$-term approximation} of a function $f\in V$ as
$$
\Sigma_M(f): = \inf_{\substack{I_M \subset I, |I_M| = M,\\ (c_i)_{i \in I_M}}} \bigg\| \sum_{i\in I_M} c_i \phi_i - f\bigg\|_V. 
$$
For $C\subset V,$ the system $(\phi_i)_{i\in I}$ yields an \textit{$M$-term approximation rate} of $M^{-r}$ for $r\in \R^+$ if
$$
\sup_{f\in C}\Sigma_M(f) = \lesssim M^{-r} \text{ for } M \to \infty.
$$
\end{definition}
It is possible to show that certain shearlet systems yield almost optimal $M$-term approximation rates\footnote{The optimal $M$-term approximation rate is the best rate that can be achieved under some restrictions on the representation system and the selection procedure of the coefficients. See~\cite{DCartoonLikeImages2001} for optimal $M$-term approximation rates for cartoon-like functions.} for the class of cartoon-like functions $\mathcal{E}^n([0,1]^2)$ (see, for instance, \cite{ShearletsCompact}). 

In~\cite[Theorem.~6.8]{boelcskeiNeural}, (optimal) $M$-term approximation rates of shearlets are transferred to $M$-weight approximations with neural networks. It is shown that with certain assumptions on the activation function $\sigma$ one can emulate a generating function $\psi$ with a fixed size neural network $\Phi_{\psi}$ such that $\psi\approx \act{\Phi_\psi}$. As a consequence, for every element $\phi_i$ of the system $\mathcal{SH}$ there exists a corresponding fixed size neural network $\Phi_i$ with $\phi_i\approx\act{\Phi_i}$. An $M$-term approximation $\sum_{i\in I_M} c_i(f) \phi_i$ of a function $f$ can then be approximated by a parallelization of networks $\Phi_i$ with $\lesssim M$ weights. This line of arguments was first used in~\cite{shaham2018provable} and also works for general affine systems. This is made precise in the next theorem.
\begin{theorem}[\cite{boelcskeiNeural}]\label{thm:affdicopt}
Let $K\subset \mathbb{R}^d$ be bounded and $\mathcal{D}=(\varphi_i)_{i\in \mathbb{N}}\subset \mathcal{L}_2(K)$ an affine system with generating function $\psi\in \mathcal{L}_2(K)$. Suppose that for an activation function $\sigma:\mathbb{R}\to \mathbb{R}$ there exists a constant $C\in\N$ such that for all $\epsilon>0$ and all $D>0$ there is a neural network
    $\Phi_{D,\epsilon}$ with at most $C$ non-zero weights satisfying
    \begin{equation*}
        \|\psi - \act{\Phi_{D,\epsilon}}\|_{\mathcal{L}_2([-D,D]^d)} \leq \epsilon.
    \end{equation*}
     Then, if $\epsilon>0$, $M\in \N$, $g\in \mathcal{L}_2(K)$ such that there exist $(d_i)_{i = 1}^M$: 
    $$
    \left\|g - \sum_{i = 1}^M d_i \varphi_i\right\|_{2} \leq \epsilon, 
    $$
    there exists a neural network $\Phi$ with $\lesssim M$ nonzero weights such that 
    $$
    \|g -\act{\Phi}\|_{2} \leq 2\epsilon.
    $$
\end{theorem}
Consequently, if shearlet systems yield a certain $M$-term approximation rate for a function class $\mathcal{C}$, then neural networks produce at least that error rate in terms of weights.
We can conclude from Theorem~\ref{thm:affdicopt} that neural networks yield an optimal $M$-weight approximation rate of $\lesssim M^{-n/2}$. On the other hand, we saw in Example~\ref{example:optimal_coefficient} that $\gamma^*(\mathcal{E}^n([0,1]^2)) = 2/n$ so that Theorem \ref{thm:optimalitynoquant} demonstrates that $\lesssim M^{-n/2}$ is also the optimal approximation rate.
Similar results have been deduced in \cite{GrohsPEB2019}.

An extension to functions $f\in \mathcal{L}_p([-1/2,1/2]^d),~p\in (0,\infty)$ that are $C^\beta$-smooth apart from $C^\beta$-singularity hypersurfaces is derived in~\cite{petersen2018equivalence}.
It is shown that the Heaviside function is approximable by a shallow ReLU neural networks with five weights (\cite[Lemma A.2.]{PetV2018OptApproxReLU}). A combination with Theorem~\ref{thm:PETVOIGTSMOOTH} yields then the next theorem.

\begin{theorem}[\cite{PetV2018OptApproxReLU}]\label{thm:PETVOIGTPIECEWISESMOOTH}
 Let $\beta=(n,\zeta),~n\in \N_0,~\zeta\in (0,1]$ and $p\in(0,\infty).$ Let $f= \mathbf{1}_K\cdot g$, where we assume that $g\in C^{\beta'}([-1/2,1/2]^d)$ for $\beta'= (d\beta)/(p(d-1))$ with $\|g\|_{C^{\beta'}}\leq 1$ and we assume $K\subset[-1/2,1/2]^d$ with $\partial K\in C^\beta$. Moreover, let $\sigma=\mathrm{ReLU}.$
 
 Then, for every $\epsilon>0$ there exist a neural network $\Phi_{f,\epsilon}$ with encodable weights, $L(\Phi_{f,\eps}) \lesssim\log_2((n+\zeta))\cdot (n+\zeta)/d$ layers as well as $M(\Phi_{f,\epsilon}) \lesssim \epsilon^{-p(d-1)/(n+\zeta)} $ non-zero weights such that 
    \[
    \left\|f-\act{\Phi_{f,\epsilon}} \right\|_{p}\leq \eps.
    \]
\end{theorem}

This rate is also shown to be optimal. We note, that approximation rates for piecewise H{\"o}lder functions in $\mathcal{L}_2$ have been proven in~\cite{imaizumi2018deep} and more general spaces like Besov spaces have been considered in~\cite{suzuki2018adaptivity}.

Some remarks on the role of depht for the results presented above can be found in Section~\ref{sec:ComparisonDeepShallow}.

\section{Assuming More Structure} \label{sec:Structure}

We have seen in the last sections that even approximations by deep neural networks face the curse of dimensionality for classical function spaces. How to avoid this by assuming more structured function space is the topic of this section. 
We start in Subsection~\ref{subsec:Hierarch} by assuming a hierarchical structure for which deep neural networks overcome the curse of dimensionality but shallow ones do not. Afterwards, we review approximations of high-dimensional functions lying on a low-dimensional set in Subsection~\ref{subsec:DataManifold}. Of a similar flavor are the results of Subsection~\ref{subsec:PDEApplications} where we examine specifically structured solutions of (parametric) partial differential equations. 

\subsection{Hierachical Structure} \label{subsec:Hierarch}

In~\cite{liang2016deep}, tight approximation rates for smooth functions $f:[0,1]^d\to \R,~x\mapsto g_1\circ g_2\circ\dots \circ g_k\circ l(x)$ with a hierarchical structure, where $l$ is a multivariate polynomial and $g_1,\dots,g_k$ are sufficiently smooth univariate functions, have been derived. Achieving a uniform approximation error $\epsilon>0$ with $L \lesssim 1$ layers requires $N \gtrsim \mathrm{poly}(1/\epsilon)$ neurons, whereas neural networks with $L \sim 1/\epsilon$ layers only require $N\lesssim \mathrm{polylog}(1/\epsilon)$ neurons. The proof idea is again based on the approximation of Taylor polynomials.

Similar to \cite{liang2016deep}, the papers \cite{mhaskar2016learning,mhaskar2016deep,mhaskar2017when,poggio2017and} deduce superior deep neural network approximation rates for high-dimensional functions with a compositional structure. They argue that computations on e.g.\ images should reflect properties of image statistics such as locality and shift invariance, which naturally leads to a compositional structure. As an example, consider a function $f:[-1,1]^d\to \R$ with input dimension $d=8$ such that 
\[
f(x_1,\dots,x_8)= f_3(f_{21}(f_{11}(x_1,x_2),f_{12}(x_3,x_4)) , f_{22}(f_{13}(x_5,x_6),f_{14}(x_7,x_8))
),\]
where each of the functions $f_3,f_{21},f_{22},f_{11},f_{12},f_{13},f_{14}$ is bivariate and in $W^{n,\infty}([-1,1]^2)$ (see Figure \ref{fig:MhaskarTree} for a visualization). Efficient approximations can be constructed in two steps. First, each bivariate function is approximated by a shallow neural network with smooth, non-polynomial activation function and size $M\lesssim\epsilon^{-2/n}$. Then, the neural networks are concatenated in a suitable way by allowing depth $L \lesssim \log_2(d)$. The resulting neural network has $M\lesssim(d-1)\epsilon^{-2/n}$ weights (see \cite[Theorem 2]{poggio2017and}). Contrary, shallow neural networks with the same activation function require $M \gtrsim \epsilon^{-d/n}$ parameters (see \cite[Theorem 1]{poggio2017and}). Moreover, if the components $f_{ij}$ of $f$ are simply assumed to be Lipschitz continuous, then shallow ReLU neural networks in general require $M \gtrsim \epsilon^{-d}$ parameters. On the other hand, deep ReLU neural networks with $L\lesssim\log_2(d)$ layers only require $M\lesssim (d-1)\epsilon^{-2}$ non-zero weights (see \cite[Theorem 4]{poggio2017and}).
\begin{figure}
\centering
\begin{forest} 
for tree={edge={},l+=1cm}
[,label=$f_3$,regular
    [,label={[label distance=1]0:$f_{21}$},regular
        [,label={[label distance=1]0:$f_{11}$},regular
        [$x_1$]
        [$x_2$]
        ]
        [,label={[label distance=1]0:$f_{12}$},regular
        [$x_3$]
        [$x_4$]
        ]
    ]
    [,label={[label distance=1]0:$f_{22}$},regular
        [,label={[label distance=1]0:$f_{13}$},regular
        [$x_5$]
        [$x_6$]
        ]
        [,label={[label distance=1]0:$f_{14}$},regular
        [$x_7$]
        [$x_8$]
        ]
    ]
]
\end{forest}
\caption{A visualization of the hierarchically structured function $f$.} \label{fig:MhaskarTree}
\end{figure}
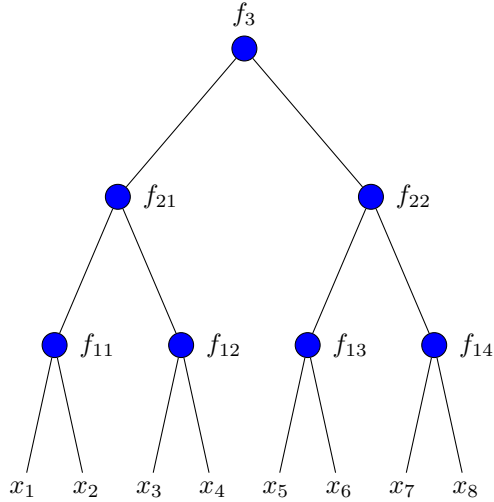

In~\cite{montanelli2019new}, neural network approximation for the \emph{Korobov spaces} 
\begin{align*}
    \mathcal{K}_{2,p}([0,1]^d) = \left\{f\in \mathcal{L}_p([0,1]^d):f|_{\partial [0,1]^d}= 0, D^{\alpha} f\in \mathcal{L}_p([0,1]^d),\|\alpha\|_\infty \leq 2 \right\},
\end{align*}
for $p\in [2,\infty]$ are considered. We note that trivially $\mathcal{K}_{2,p}([0,1]^d) \subset W^{2,p}([0,1]^d)$. These functions admit a hierarchical representation (similar to a wavelet decomposition) with respect to basis functions obtained from \emph{sparse grids} (see \cite{SparseGrids}). By emulating these basis functions, the authors are able to show that for every $f\in \mathcal{K}_{2,p}([0,1]^d)$ and every $\epsilon>0$ there exists a neural network $\Phi_{f,\epsilon}$ with $$L(\Phi_{f,\epsilon})\lesssim\log_2(1/\epsilon) \log_2(d)  $$ layers as well as $$N(\Phi_{f,\epsilon})\lesssim\epsilon^{-1/2} \cdot \log_2(1/\epsilon)^{\frac{3}{2}(d-1)+1} \log_2(d)$$ neurons such that
\[\left\|f-\act{\Phi_{f,\epsilon}} \right\|_\infty \leq \epsilon,\]  where $\sigma=\mathrm{ReLU}$.
The curse of dimensionality is significantly lessened since it only appears in the log factor.

\subsection{Assumptions on the Data Manifold} \label{subsec:DataManifold}
A typical assumption is that high-dimensional data actually resides on a much lower dimensional manifold. A standard example is provided in Figure~\ref{fig:SwissRoll}. One may think about the set of images with $128\times 128$ pixels in $\R^{128\times 128}$: certainly, most elements in this high-dimensional space are not perceived as an image by a human, making images a proper subset of $\R^{128\times 128}$. Moreover, images are governed by edges and faces and, thus, form a highly structured set. This motivates the idea that the set of images in $\R^{128\times 128}$ can be described by a lower dimensional manifold.
\begin{figure}
    \centering
    \includegraphics[width=.4\textwidth]{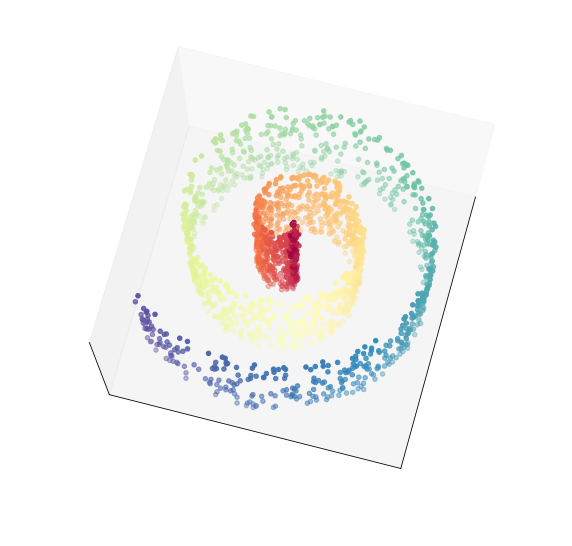}
    \caption{Visualization of the Swiss Roll Data Manifold }
    \label{fig:SwissRoll}
\end{figure}

In~\cite{shaham2018provable} approximation rates of ReLU neural networks for functions $f\in \mathcal{L}_2(K)$ residing on a $d$-dimensional manifold $K\subset \R^D$ with possibly $d\ll D$, are shown to only weakly depend on $D$. In detail, the following theorem was proven.
\begin{theorem}[\cite{shaham2018provable}]\label{thm:wavelet_manifold}
Let $K\subset\R^D$ be a smooth $d$-dimensional manifold, and $f \in \mathcal{L}_2(K)$. Then there exists a depth-4 neural network $\Phi_{f,N}$ with $M(\Phi_{f,N})\lesssim D C_K + d C_K N$ whose ReLU realization computes a wavelet approximation of $f$ with $N$ wavelet terms. The constant $C_K$ only depends on the curvature of $K$.
\end{theorem}
Note that the underlying dimension $D$ scales with $C_K$, whereas the number of wavelet terms that influences the approximation accuracy only scales with $d$. Additional assumptions on how well $f$ can be approximated by wavelets, can with Theorem~\ref{thm:wavelet_manifold} directly be transferred to approximation rates by neural networks. The following corollaries provide two examples for different assumption on the wavelet representation. 
\begin{corollary}
If $f$ has wavelet coefficients in $\ell_1,$ then, for every $\eps>0,$ there exists a depth-4 network $\Phi_{f,\eps}$ with $M(\Phi_{f,\eps})\lesssim D C_K + d C_K^2 M_f\eps^{-1}$ such that \[
\norm{f-\act{\Phi_{f,\eps}}}_{\mathcal{L}_2(K)}\leq\eps,
\]
where $M_f$ is some constant depending on $f$ and $\sigma = \mathrm{ReLU}$. 
\end{corollary}

\begin{corollary}
If $f\in C^2(K)$ has bounded Hessian matrix, then, for every $\eps>0,$ there exists a depth-4 neural network $\Phi_{f,\eps}$ with $M(\Phi_{f,\eps}) \lesssim D C_K + d C_K \eps^{-d/2}$ satisfying \[
\norm{f-\act{\Phi_{f,\eps}}}_{\mathcal{L}_\infty(K)}\leq\eps,
\]
where $\sigma=\mathrm{ReLU}.$
\end{corollary}

Lastly, we mention a result shown in~\cite[Theorem 5.4]{PetV2018OptApproxReLU}. There, functions of the type $f= g\circ h$ where $h:[-1/2,1/2]^D\to [-1/2,1/2]^d$ is a smooth dimension reduction map and $g:[-1/2,1/2]^d\to \R$ is piecewise smooth are examined. They show that the approximation rate is primarily governed by the reduced dimension $d$. 

\subsection{Expressivity of Deep Neural Networks for Solutions of PDEs}\label{subsec:PDEApplications}

Recently, neural network-based algorithms have shown promising results for the numerical solution of partial differential equations (PDEs) (see for instance \cite{Lagaris,weinan2017deep,weinan2018deep,sirignano2018dgm,han2018solving,beck2018solving,Han,SplittingJentzen,Khoo, RBNonlinearProblems, lee2018model, yang2018physics, raissi2018deep,DeepXDE}).

There also exist several results showing that specifically structured solutions of PDEs admit an approximation by neural networks that does not suffer from the curse of dimensionality (see \cite{JentzenHeat,grohs2018proof,SchwabOption,berner2018analysis,Reisinger2019,JentzenHeat,welti,PricingJentzen}).

The key idea behind these contributions is a stochastic interpretation of a deterministic PDE\footnote{This stochastic point of view can also be utilized to estimate the generalization error in such a setting (\cite{berner2018analysis}).}. As an example, we consider the \emph{Black-Scholes Equation}, which models the price of a financial derivative. For $T>0,~a<b$ the Black-Scholes Equation (a special case of the \emph{linear Kolmogorov Equation}) is given by
\begin{align} \label{eq:BlackScholes}
    \begin{cases} \partial_t u(t,x)= \frac{1}{2} \mathrm{trace}\left(\kappa(x)\kappa(x)^*(\mathrm{Hessian}_x u)(t,x) \right) + \left\langle \tilde{\kappa}(x), (\nabla_x u)(t,x)\right\rangle, \\
    u(0,x) = \varphi(x), \end{cases}
\end{align}
where 
\begin{itemize}
    \item[(i)] $\varphi\in C(\R^d)$ is the \emph{initial value}\footnote{The initial value is typically either exactly representable or well-approximable by a small ReLU neural network.},
    \item[(ii)] $\kappa\in C(\R^d,\R^{d\times d}),~\tilde{\kappa}\in C(\R^d,\R^d)$ are assumed to be affine, and
    \item[(iii)] $u\in C([0,T]\times \R^d)$ is the solution.    
\end{itemize}
In \cite{berner2018analysis} the goal is to find a neural network (of moderate complexity) that $\mathcal{L}_2$-approximates the end value of the solution $[a,b]^d\ni x\mapsto u(T,\cdot).$ 
Using the Feynman-Kac formula \cite[Section 2]{grohs2018proof}, the deterministic PDE~\eqref{eq:BlackScholes} can be associated with a stochastic PDE of the form
\begin{align} \label{eq:StochPDE}
    dS_t^{x} = \kappa(S_t^x)dB_t + \tilde{\kappa}(S_t^x)dt,\quad S_0^x=x,
\end{align}
on some probability space $(\Omega,\mathcal{G},\mathbb{P}),$ where $x \in [a,b]^d$, $(B_t)_{t\in [0,T]}$ is a $d$-dimensional Brownian motion and $(S_t^x)_{t\in [0,T]}$ the stochastic process that solves \eqref{eq:StochPDE}. Define $Y\coloneqq \varphi(S^x_t)$. Then the terminal state $u(T,x)$ has an integral representation with respect to the solution of \eqref{eq:StochPDE} given by $u(T,x) = \mathbb{E}(Y)$. Based on this relation, standard estimates on Monte-Carlo sampling and the special structure of $\kappa,\tilde{\kappa}$ \cite{berner2018analysis} derive the following statement: For each $i\in \N,$ there exist affine maps $\mathbf{W}_i(\cdot)+\mathbf{b}_i:~\R^d\to \R$ such that 
\begin{align*}
    \frac{1}{(b-a)^d} \left\| u(T,\cdot)-\frac{1}{n} \sum_{i=1}^n \varphi(\mathbf{W}_i(\cdot )+\mathbf{b}_i) \right\|_2^2  \lesssim \frac{d^{1/2}}{n}.
\end{align*}
This implies that $u(T,\cdot)$ can be approximated in $\mathcal{L}_2$ by ReLU neural networks $\Phi_{u,\epsilon}$ with $M(\Phi_{u,\epsilon})\lesssim \mathrm{poly}(d,1/\epsilon)$. Thus, the curse of dimensionality is avoided.

In the setting of parametric PDEs one is interested in computing solutions $u_y$ of a family of PDEs parametrized by $y\in\R^p$. Parametric PDEs are typically modelled as operator equations in their variational form
\begin{align*}
    b_y (u_y,v) = f_y(v), \quad \text{ for all } v\in \mathcal{H},~ y\in \mathcal{P} \quad \text{ (= parameter set)},
\end{align*}
where, for every $y\in \mathcal{P}\subset \R^p,$ ($p\in \N\cup \{\infty\}$)
\begin{itemize}
    \item[(i)] the maps $b_y:\mathcal{H}\times \mathcal{H}\to \R$ are parameter-dependent bilinear forms (derived from the PDE) defined on some Hilbert space $\mathcal{H}\subset L^\infty(K)$ for $K\subset \R^n$,
    \item[(ii)] $f_y\in \mathcal{H}^*$ is the parameter-dependent right-hand side, and
    \item[(iii)] $u_y\in \mathcal{H}$ is the parameter-dependent solution.
\end{itemize} 
The parameter $y\in\R^p$ models uncertainties in real-world phenomena such as geometric properties of the domain, physical quantities such as elasticity coefficients or the distribution of sources.
Parametric PDEs occur for instance in the context of multi-query applications and in the framework of uncertainty quantification. 

In the context of deep learning, the goal is, for given $\epsilon>0,$ to substitute the solution map $y\mapsto u_y$ by a neural network $\Phi_{\epsilon}$ such that 
\begin{align} \label{eq:ApproxPPDE}
    \sup_{y\in\mathcal{P}}\left\|\act{\Phi_{\epsilon}(y,\cdot)}- u_y\right\|_{\mathcal{H}} \leq \epsilon.
\end{align}
A common observation is that the \emph{solution manifold} $ \{u_y:y\in \mathcal{P}\}$ is low-dimensional for many parametric problems. More precisely, there exist $(\varphi_i)_{i=1}^d \subset \mathcal{H},$ where $d$ is comparatively small compared to the ambient dimension, such that for every $y\in \Ycal$ there exists some coefficient vector $(c_i(y))_{i=1}^d$
with 
\[ \left\|u_y - \sum_{i=1}^d c_i(y) \varphi_i \right\|_{\mathcal{H}} \leq \epsilon.  \]

For analytic solution maps $y\mapsto u_y$ \cite{schwab2018deep} constructed neural networks with smooth or ReLU activation function that fulfill~\eqref{eq:ApproxPPDE}. Exploiting a sparse Taylor decomposition of the solution with respect to the parameters $y$,
they were able to avoid the curse of dimensionality in the complexity of the approximating networks.

In~\cite{NNParametric} the following oberservation was made: If the forward maps $y\mapsto b_y(u,v),~y\mapsto f_y(v)$ are well-approximable by neural networks for all $u,v\in \mathcal{H},$ then also the map $y\mapsto (c_i(y))_{i=1}^d$ is approximable by ReLU neural networks $\Phi_{c,\epsilon}$ with $M(\Phi_{c,\epsilon})\lesssim \mathrm{poly}(d)\cdot \mathrm{polylog}(1/\epsilon) $. Since in many cases $d \lesssim\log_2(1/\epsilon)^p$ and for some cases one can even completely avoid the dependence of $d$ on $p,$ the curse of dimensionality is either significantly lessened or completely overcome.  
The main idea for the computation of the coefficients by neural networks lies in the efficient approximation of the map $y\mapsto ((b_y(\varphi_j,\varphi_i))_{i=1}^d)^{-1}$. This is done via a Neumann series representation of the matrix inverse, which possesses a hierarchical structure.

\section{Deep versus Shallow Neural Networks} \label{sec:ComparisonDeepShallow}
\emph{Deep} neural networks have advanced the state of the art in many fields of application. 
We have already seen a number of results highlighting differences between the expressivity of shallow and deep neural networks throughout the paper. In this section, we present advances made in approximation theory explicitly aiming at revealing the role of depth.

First of all, we note that basically all upper bounds given in Section~\ref{sec:ApproxSmooth} and Section~\ref{sec:DeepPiecewiseRates} use $L>2$ layers. Furthermore, in~\cite[Section 4.4]{YAROTSKY2017103} and~\cite[Section 4.2]{PetV2018OptApproxReLU}, it is shown that the upper bound of non-zero weights in Theorem~\ref{thm:yarotsky}, Theorem~\ref{thm:PETVOIGTSMOOTH} and Theorem~\ref{thm:PETVOIGTPIECEWISESMOOTH} \emph{cannot} be achieved by shallow networks. For a fixed number of layers, \cite{PetV2018OptApproxReLU} showed that the number of layers in Theorem~\ref{thm:PETVOIGTSMOOTH} and Theorem~\ref{thm:PETVOIGTPIECEWISESMOOTH} is optimal up to the involved log-factors.

A common strategy in many works which compare the expressive power of very deep over moderately deep neural network is to construct functions that are efficiently approximable by deep neural networks, whereas far more complex shallow neural networks are needed to obtain the same approximation accuracy. The first results in this direction were dealing with the approximation of Boolean functions by \emph{Boolean circuits} (see, e.g., \cite{Sipser1983,hastad1986almost}). In~\cite{hastad1986almost}, the existence of function classes that can be computed with Boolean circuits of polynomial complexity and depth $k$ is proven. In contrast, if the depth is restricted to $k-1$, an exponential complexity is required. Later works such as \cite{haastad1991power, hajnal1993threshold, martens2014expressive} and the references therein focused on the representation and approximation of Boolean functions by deep neural networks where the activation is the threshold function $\sigma=\mathbf{1}_{[0,\infty)}.$\footnote{called \emph{deep threshold circuits}}

In~\cite{delalleau2011shallow}, networks built from sum and product neurons, called \emph{sum-product} networks, are investigated. A sum neuron computes a weighted sum\footnote{as in the standard setting} and a product neuron the product of its inputs. The identity function is used as activation function and layers of product neurons are alternated with layers of sum neurons. The authors then considered functions implemented by deep sum-product networks with input dimension $d$, depth $L\lesssim\log_2(d)$ and $N\lesssim d$ neurons that can only be represented by shallow sum-product networks with at least $N\gtrsim 2^{\sqrt d}$ neurons. Thus, the complexity of the shallow representations grows exponentially in $\sqrt d$. Similar statements are shown for slightly different settings.

\cite{PowerOfDepth} examines approximation rates of neural networks with smooth activation functions for polynomials. It has been shown that in order to approximate a polynomial by a shallow neural network, one needs exponentially more parameters than with a corresponding deep counterpart. 

In~\cite{Telgarsky2016BenefitsOfDepth} neural networks with $L$ layers are compared with networks with $L^3$ layers. Telgarsky shows that there exist functions $f:\R^d\to \R$ that can be represented by neural networks with $\sim L^3$ layers and with $\sim L^3$ neurons which, in $\mathcal{L}_1,$ cannot be arbitrarily well approximated by neural networks with $\lesssim L$ layers and $\lesssim  2^L$ neurons. This result is valid for a large variety of activation functions such as piecewise polynomials. The main argument of the paper is based on the representation and approximation of functions with oscillations. First, it is shown that functions with a small number of oscillations cannot approximate functions with a large number of oscillations arbitrarily well. Afterwards, it is shown that functions computed by neural networks with a small number of layers can only have a small number of oscillations, whereas functions computed by deeper neural networks can have many oscillations. A combination of both arguments leads to the result. The papers \cite{eldan2016power,pmlr-v70-safran17a} show similar results for a larger variety of activation functions\footnote{including the ReLU and sigmoidal functions} in which the $\mathcal{L}_2$-approximation of certain radial basis functions\footnote{i.e. functions of the form $f:\R^d\to \R$ where $f(x)=g(\|x\|_1)$ for some univariate function $g$} requires $M\lesssim 2^d$ weights, whereas three layered networks can exactly represent these functions with $M\lesssim \mathrm{poly}(d)$ weights. 

In~\cite{pascanu2013number,montufar2014number} it is shown that deep ReLU neural networks have the capability of dividing the input domain into an exponentially larger number of linear regions than shallow ones.\footnote{A \emph{linear region} of a function $f:\R^d\to \R^s$ is a maximally connected subset $A$ of $\R^d$ such that $f|_A$ is linear.} A ReLU neural network with input dimension $d,$ $L$ layers and width $n\geq d$ is capable of computing functions with number of linear regions less than $(n/d)^{(L-1)d} n^d $, whereas ReLU neural networks with input dimension $d$, two layers and width $Ln$ are only able to divide the input domain into $L^d n^d$ linear regions. Similarly, the paper \cite{ExpressivePowerDNN} gives corresponding upper bounds on the number of linear regions representable by deep ReLU neural networks. It is shown that ReLU neural networks with input dimension $d,$ $L$ layers and width $n,$ are only able to represent $\lesssim n^{Ld}$ linear regions.

Comparable results like those of the aforementioned papers are the statements of \cite{arora2018understanding}. In particular, \cite[Theorem 3.1]{arora2018understanding} establishes that for every $L\in \N,~K\geq 2$ there exists a function $f:\R\to\R$ which is representable by ReLU neural networks with $L+1$ layers, and $K\cdot L$ neurons. Moreover, if $f$ is also representable by the ReLU realization of a neural network $\Tilde{\Phi}$ with $\Tilde{L}+1\leq L+1$ layers, then the number of neurons of $\Tilde{\Phi}$ is bounded from below by $\gtrsim \Tilde{L}K^{L/\Tilde{L}}.$ 
The paper \cite{complexity2014bianchini} connects the approximative power of deep neural networks to a topological complexity measure based on Betti numbers: For shallow neural networks, the representation power grows polynomially in the numbers of parameters, whereas it grows exponentially for deep neural networks. Also, connections to complexity measures based on VC-dimensions are drawn.

\section{Special Neural Network Architectures and Activation Functions}\label{sec:SpecialArchitecture}

So far we entirely focused on general \emph{feedforward} neural network architectures. But since neural networks are used for many different types of data in various problem settings, there exists a plethora of architectures, each adapted to a specific task. In the following we cover expressivity results for three prominent architectures: \emph{Convolutional neural networks, residual neural networks and recurrent neural networks}. Since explaining the architectures and their applications in detail is beyond the scope of this paper we will only briefly review the basics and give references for the interested reader.\footnote{Note that in this section we sometimes deviate from the notation used so far in this paper. For instance, in the case of convolutional neural networks, we do not differentiate between the network as a collection of weights and biases and function realized by it anymore. Such a distinction would make the exposition of the statements in these cases unnecessarily technical. Moreover, we closely stick to the conventional notation for recurrent neural networks where the indices of the layers are expressed in terms of discrete time steps $t$.}

\subsection{Convolutional Neural Networks}
Convolutional neural networks have first been introduced in \cite{lecun1989generalization} and since then lead to tremendous successes, in particular, in computer vision tasks. 
As an example, consider the popular ILSVR\footnote{The abbreviation \emph{ILSRC} stands for \emph{ImageNet Large Scale Visual Recognition Challenge}.} challenge (see~\cite{russakovsky2015imagenet}), an image recognition task on the ImageNet database (see~\cite{deng2009imagenet}) containing variable-resolution images that are to be classified into categories. In~2012 a convolutional neural network called \emph{AlexNet} (see~\cite{AlexNet}) achieved a top-5 error of 16.4~\% realising a 10 \% error rate drop compared to the winner of the previous year. The winners of all succeeding annual ILSVR challenges have until now been convolutional neural networks  (see \cite{zeiler2014visualizing, simonyan2014very,szegedy2015going,he2016deep},\dots).

We now start with a brief explanation of the basic building blocks of a convolutional neural network and recommend \cite[Chapter\ 9]{Goodfellow} for an extensive and in-depth introduction.

Instead of vectors in $ \R^d,$ the input of a convolutional neural network potentially consists of tensors $\mathbf{x}^{[0]}\in \R^{d_1\times\dots\times d_n}$ where $n=2$ or $n=3$ for the case of images described above. The main characteristic of convolutional neural networks is the use of \emph{convolutional layers}. The input $\mathbf{x}^{[i-1]}$ of layer $i$ is subject to an affine transformation of the form 
\[\mathbf{z}^{[i]}= \mathbf{w}^{[i]}\ast\mathbf{x}^{[i-1]}+\mathbf{b}^{[i]},\]
where $\mathbf{w}^{[i]}$ denotes the \emph{convolution filter} and $\mathbf{b}^{[i]}$ the \emph{bias}. In fact, the application of a convolution can be rewritten as an affine transformation of the form $\mathbf{z}^{[i]}= \mathbf{W}^{[i]}\mathbf{x}^{[i-1]}+\mathbf{b}^{[i]},$ where $\mathbf{W}^{[i]}$ is a matrix with a specific structure and potentially many zero entries (depending on the size of the filter). Convolutional layers enforce \emph{locality} in the sense that neurons in a specific layer are only connected to neighboring neurons in the preceeding layer, as well as \emph{weight sharing}, i.e. different neurons in a given layer share weights with neurons in other parts of the layer. The application of convolutions results in \emph{translation equivariant} outputs in the sense that $\mathbf{w}^{[i]}\ast (T \mathbf{x}^{[i-1]}) = T(\mathbf{w}^{[i]}\ast \mathbf{x}^{[i-1]})$ where $T$ is a translation (or shift) operator. 

As in the case of feedforward neural networks, the convolution is followed by a nonlinear activation function $\sigma: \R\to \R$, which is applied coordinate-wise to $\mathbf{z}^{[i]}.$ Hence, a convolutional layer corresponds to a standard layer in a feedforward neural network with a particular structure of the involved affine map. 

Often the nonlinearity is followed by a \emph{pooling layer}, which can be seen as a dimension reduction step. Popular choices for pooling operations include \emph{max pooling,} where only the maxima of certain regions in $\sigma(\mathbf{z}^{[i]})$ are kept, or \emph{average pooling}, where the average over certain regions is taken. Contrary to convolutions, pooling often induces \emph{translation invariance} in the sense that a small translation of the input $\mathbf{x}^{[i-1]}$ does not change the output of the pooling stage.

Stacking several convolutional and pooling layers results in a deep convolutional neural network. After having fixed an architecture, the goal is to learn the convolutional filters $\mathbf{w}^{[i]}$ and the biases $\mathbf{b}^{[i]}.$

We now turn our attention to expressivity results of convolutional neural networks.
The paper \cite{CohenSharir} compares the expressivity of deep and shallow \emph{convolutional arithmetic circuits} for the representation of \emph{score functions,} motivated by classification tasks. A convolutional arithmetic circuit can be interpreted as a convolutional neural network with linear activation function, product pooling, and convolutional layers consisting of $1\times1$ convolutions. We consider a classification task of the following form: For an observation $X=(\mathbf{x}_1,\dots,\mathbf{x}_m)\in \R^{d \cdot  m},$ where $X$ could be an image represented by a collection of vectors, find a suitable classification $n \in Y = \{1,\dots,N\}$ where $Y$ is the set of possible labels. This can be modelled by per-label score functions $\{h_n\}_{n\in \{1,\dots,N\}},$ with $h_n:~\R^{d\cdot m}\to \R$. The label corresponding to $X$ is then denoted by $\argmax_{n\in Y}h_n(X)$. It is assumed that these score functions have a tensor decomposition of the form 
\begin{align*}
    h_n(\mathbf{x}_1,\dots,\mathbf{x}_m)= \sum_{i_1,\dots,i_m=1}^r \mathbf{c}^n_{i_1,\dots,i_m} \prod_{j=1}^m f_{{i_j}}(\mathbf{x}_i)
\end{align*}
for some representation functions $f_{1},\dots,f_{r}:\R^d\to \R$ and a coefficient tensor $\mathbf{c}^n.$
Hence, every $h_n$ is representable by a \emph{shallow} convolutional arithmetic circuit, however, with a \emph{potentially large number} of parameters. Using hierarchichal tensor decompositions, the main result states, roughly speaking, that it is with high probability not possible to significantly reduce the complexity of the shallow convolutional arithmetic circuit in order to approximate or represent the underlying score function. However, \emph{deep} convolutional arithmetic circuits \emph{are} able to exactly represent score functions with exponentially less parameters. 
This is yet another instance of the benefit of deep over shallow neural networks, this time for a larger function class rather than only for special instances of functions as the ones considered in Subsection \ref{sec:ComparisonDeepShallow}. The main proof ideas are relying on tools from matrix algebra, tensor analysis, and measure theory.

In \cite{YarotskyCNNs} the classification task is approached by directly approximating the classification function $f:\mathcal{L}_2(\R^2)\to\R$, which maps an image to a real number. Here, images are modelled as elements of the function space $\mathcal{L}_2(\R^2)$ (see also Subsection \ref{sec:DeepPiecewiseRates}, where images are modelled by cartoon-like functions). In order to deal with a non-discrete input (in this case elements of $\mathcal{L}_2$), the neural networks considered in this paper consist of a discretization step $\mathcal{L}_2(\R^2)\to V$ (where $V$ is isometrically isomorphic to $\R^{D_\eps}$ for some $D_\eps<\infty$) followed by convolutional layers and downsampling operations, which replace pooling. The following theorem \cite[Theorem~3.2]{YarotskyCNNs} was shown:

\begin{theorem}
Let $f:\mathcal{L}_2(\R^2)\to\R$. Then, the following conditions are equivalent:
\begin{itemize}
    \item[(i)] The function $f$ is continuous (in the norm topology).
    \item[(ii)] For every $\eps>0$ and every compact set $K\subset\mathcal{L}_2(\R^2)$ there exists a convolutional neural network (in the above sense, with downsampling) $\Phi_{K,\eps}$ that approximates $f$ uniformly on $K,$ i.e.,\
    \[
    \sup_{\xi\in K}\pabs{f(\xi)-\Phi_{K,\eps}(\xi)}\leq\eps.\]
\end{itemize}
\end{theorem}

A second setting considered in \cite{YarotskyCNNs} deals with approximating translation equivariant image to image mappings. Think for example of a seqmentation task where an image (e.g. of a cell) is mapped to another image (e.g. the binary segmentation mask). Translating the input image should result in a translation of the predicted segmentation mask, which means in mathematical terms that the mapping is translation equivariant. To make also the convolutional neural networks translation equivariant no downsampling is applied. 
The next theorem \cite[Theorem~3.1]{YarotskyCNNs} addresses the second setting.
\begin{theorem}[simplified] \label{thm:YarotskyEquivariant}
Let $f:\mathcal{L}_2(\R^2)\to\mathcal{L}_2(\R^2)$. Then, the following conditions are equivalent:
\begin{itemize}
    \item[(i)] The function $f$ is continuous (in the norm topology) and translation equivariant, i.e.\ $f(\xi(\cdot-\tau)) = f(\xi)(\cdot -\tau)$ for all $\tau\in\R^2$.
    \item[(ii)] For every $\eps>0$ and every compact set $K\subset\mathcal{L}_2(\R^2),$ there exists a convolutional neural network (in the above sense) $\Phi_{K,\eps}$, which approximates $f$ uniformly on $K$, i.e., \
    \[
    \sup_{\xi\in K}\norm{f(\xi)-\Phi_{K,\eps}(\xi)}_{\mathcal{L}_2}\leq\eps.\]
\end{itemize}
\end{theorem}
We note that \cite[Theorem~3.1]{YarotskyCNNs} considers equivariances with respect to more general transformations, including rotations of the function domain. 

A result of a similar flavor as Theorem \ref{thm:YarotskyEquivariant} for functions $f:\R^d\to\R^s$ that are equivariant with respect to finite groups has been shown in \cite{petersen2018equivalence}. In this paper, it is established that every fully-connected neural network can be expressed by a convolutional neural network without pooling and with periodic padding (in order to preserve equivariance) with a comparable number of parameters and vice versa. This result can then be used to transfer approximation rates by fully-connected neural networks for a function class $\mathcal{C}$ to convolutional neural networks approximating a subclass $\mathcal{C}^{\text{equi}}$ containing equivariant functions. As an example, we consider translation equivariant H\"older functions $f\in \mathcal{C}^{\text{equi}}=C^\beta([-1/2,1/2]^d)$ with $\beta=(n,\zeta)$. By \cite[Proposition 4.3]{petersen2018equivalence}, there exist convolutional neural networks $\Phi_{f,\epsilon}$ that $\epsilon$-approximate $f$ in $\mathcal{L}_p([-1/2,1/2]^d)$ and have $M\lesssim \epsilon^{-d/(n+\zeta)}$ parameters and $L\lesssim \log_2((n+\zeta))\cdot (n+\zeta)/d$ layers. Note, that this rate coincides with the one provided by Theorem \ref{thm:PETVOIGTSMOOTH}.

The paper \cite{Zhou19} mainly derives two types of results: First, \cite[Theorem A]{Zhou19} establishes universality of deep purely convolutional networks (i.e.\ no pooling is applied) mapping from $\R^d\to \R$ for functions in $C(K),~K\subset \R^d$ compact. To be more precise, it is shown that, for every $f\in C(K),$ and every $\epsilon>0$, there exist some $L\in \N$ and some convolutional neural network $\Phi_{f,\epsilon}$ equipped with $L$ convolutional layers and the ReLU activation function, such that 
\begin{align*}
    \left\|f-\Phi_{f,\eps} \right\|_{\infty} \leq \eps.
\end{align*}
Note that contrary to most classical universality results, the depth does not remain uniformly bounded over the whole function class $C(K).$

Second, \cite[Theorem B]{Zhou19} establishes approximation \emph{rates} for Sobolev functions by convolutional neural networks, which we depict in a simplified form. 
\begin{theorem}[\cite{Zhou19}]
Let $r>d/2 + 2$. Then, for every $\eps>0$ and every $f\in F^{2}_{r,d}$, there exists a convolutional neural network $\Phi_{f,\eps}$ (without pooling layers) with depth $L\lesssim \eps^{-\frac{1}{d}}$ such that
\[
\norm{f-{\Phi_{f,\eps}}}_{\infty}\leq \eps.
\]
\end{theorem}

Contrary to most expressivity results considered before, the paper \cite{LossCNNsFinite} examines the ability of deep convolutional neural networks to fit a \emph{finite} set of $n$ samples $(X,Y)\in \R^{d\times n} \times \R^{s\times n}$ for some $n\in \N$. It is shown, that for such a data set there exist overparametrized convolutional neural networks (in this case the number of neurons in one of the layers is larger than the size of the data set $n$) $\Phi_{(X,Y)}$ equipped with max pooling and rather generic activation functions such that 
$$
    \sum_{i=1}^n \left|Y_i - \Phi_{(X,Y)}(X_i) \right|^2 =0.
$$

More expressivity results for convolutional neural networks can be found in~\cite{oono2018approximation}.

\subsection{Residual Neural Networks}
Driven by the observation that the performance of a neural network very much benefits from its depth, architectures with an increasing number of layers were tested. However, in~\cite{srivastava2015highway,he2015convolutional,he2016deep} the observation was made that after adding a certain number of layers test and training accuracy decreases again. The authors of~\cite{he2016deep} argue that that the drop in performance caused by adding more and more layers is not related to overfitting. Otherwise the training accuracy should still increase while only the test accuracy degrades. Knowing that deeper models should perform at least as good as shallower ones, since the additional stacked layers could learn an identity mapping, the authors conclude that the performance drop is caused by the optimization algorithm not finding a suitable weight configuration.

Motivated by those findings, a new network architecture called \emph{residual neural network (ResNet)} is proposed in~\cite{he2016deep}. It incorporates identity shortcuts such that the network only needs to learn a residual. These networks are shown to yield excellent performance (winner of several challenges) by employing a large number of layers. One of the popular models called ResNet-152 employs 152 layers, but there are also ResNets used with even more layers. To be precise, ResNets are composed of blocks of stacked layers (see Figure~\ref{fig:resnet_architecture}) having a shortcut identity connection from the input of the block to the output (Figure~\ref{fig:resnet_block}). In detail, the function implemented by such a block can be parametrized by weight matrices $\mathbf{V},\mathbf{W}\in\R^{k\times k}$, bias vector $\mathbf{b}\in \R^k$) and is given by
\[
\mathbf{z}:\R^k\to\R^k\quad\text{with}\quad \mathbf{z} (x) = \mathbf{V}\sigma(\mathbf{W}x +\mathbf{b}) + x.
\] 
Assume that those two layers are supposed to learn a function $F:\R^k\to\R^k$, then the block is forced to learn the residual $H(x)\coloneqq F(x) - x$.
The function implemented by a ResNet now results from composing residual blocks $\mathbf{z}^{[i]},~i=1,\dots,L:$
\begin{equation}\label{eq:resnet}
\text{ResNet}:\R^k\to\R^k\quad\text{with}\quad \text{ResNet}(x)=\mathbf{z}^{[L]}\circ\sigma(\mathbf{z}^{[L-1]})\circ\ldots\ldots\circ \sigma(\mathbf{z}^{[1]} x).
\end{equation}
Even if the idea behind the residual learning framework is to make the optimization process easier, we will in this paper only focus on approximation properties of the proposed architecture. 

In~\cite{hardt2016identity}, the ability of ResNets to perfectly fit a finite set of training samples is shown. The authors consider the case where the input data consists of pairs $(x_i,y_i)$ with $x_i\in\R^d$ and  $y_i\in\{e_1,\ldots,e_s\}$ for $i=1,\ldots,n$. Here, $e_j$ denotes the $j$-th unit vector in $\R^s$ such that $s$ can be interpreted as the number of classes in a classification task. The ResNet architecture considered in this paper slightly differs from the one in Equation~\eqref{eq:resnet}, in that no activation function is used after the residual blocks and some of the matrices are of different shape to allow for inputs and outputs of different dimensions. Under mild additional assumptions on the training set the following theorem is shown:
\begin{theorem}[\cite{hardt2016identity}]
    Let $n,s\in\N$ and $(x_1,y_1),\ldots, $ $(x_n, y_n)$ be a set of training samples as above. Then there exists a ResNet $\Phi$ with $M\lesssim n\log n+s^2$ weights such that 
    \[
    {\Phi}(x_i)=y_i,\quad\text{for all}\quad i=1,\ldots,n.
    \]
\end{theorem}

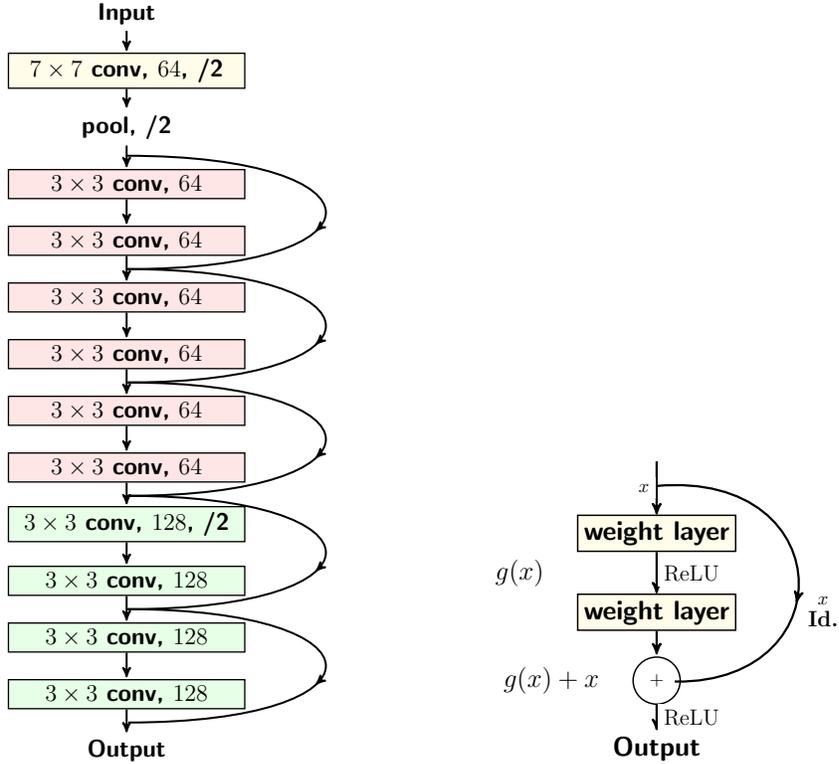
\begin{figure}[ht]
\centering
    \begin{subfigure}[b]{0.35\textwidth}
        \centering
        \resizebox{.95\linewidth}{!}{\tikzset{reg_conv_node/.style={rectangle,fill=red!10,draw,font=\sffamily\Large\bfseries,minimum width=50mm}}
\tikzset{large_conv_node/.append style={rectangle,fill=yellow!10,draw,font=\sffamily\Large\bfseries,minimum width=50mm}}
\tikzset{reg_large_conv_node/.append style={rectangle,fill=green!10,draw,font=\sffamily\Large\bfseries,minimum width=50mm}}
\tikzset{pool_node/.append style={draw=none,font=\sffamily\Large\bfseries,minimum width = 50mm}}
\begin{tikzpicture}[thick,>=stealth',shorten >=1pt,auto,node distance=12mm]
    \node[pool_node] (Input) {Input};
    \node[large_conv_node] (Entry) [below of=Input] {$7 \times 7$ conv, $64$, /2};
    \node[pool_node] (pool) [below of=Entry] {pool, /2};
    \node[reg_conv_node] (P1) [below of=pool] {$3 \times 3$ conv, $64$};
    \node[reg_conv_node] (P2) [below of=P1] {$3 \times 3$ conv, $64$};
    \node[reg_conv_node] (P3) [below of=P2] {$3 \times 3$ conv, $64$};
    \node[reg_conv_node] (P4) [below of=P3] {$3 \times 3$ conv, $64$};
    \node[reg_conv_node] (P5) [below of=P4] {$3 \times 3$ conv, $64$};
    \node[reg_conv_node] (P6) [below of=P5] {$3 \times 3$ conv, $64$};
    \node[reg_large_conv_node] (P7) [below of=P6] {$3 \times 3$ conv, $128$, /2};
    \node[reg_large_conv_node] (P8) [below of=P7] {$3 \times 3$ conv, $128$};
    \node[reg_large_conv_node] (P9) [below of=P8] {$3 \times 3$ conv, $128$};
    \node[reg_large_conv_node] (P10) [below of=P9] {$3 \times 3$ conv, $128$};
    \node[pool_node] (Output) [below of=P10] {Output};

\begin{scope}[very thick,decoration={
    markings,
    mark=at position 0.55 with {\arrow{>}}}
    ] 
\draw [postaction={decorate}] 
    ($ (pool) !.5! (P1) $) to[out=0,in=0,looseness=6] ($ (P2) !.5! (P3) $);
\draw [postaction={decorate}] 
    ($ (P2) !.5! (P3) $) to[out=0,in=0,looseness=6] ($ (P4) !.5! (P5) $);
\draw [postaction={decorate}] 
    ($ (P4) !.5! (P5) $) to[out=0,in=0,looseness=6] ($ (P6) !.5! (P7) $);
\draw [postaction={decorate}] 
    ($ (P6) !.5! (P7) $) to[out=0,in=0,looseness=6] ($ (P8) !.5! (P9) $);
\draw [postaction={decorate}] 
    ($ (P8) !.5! (P9) $) to[out=0,in=0,looseness=6] ($ (P10) !.5! (Output) $)
;

  \path[every node/.style={font=\sffamily\small,
  		fill=white,inner sep=1pt},
  		every edge/.append style={very thick,->,decoration={markings, mark=at position 1 with {\arrow{->}}}}
  		]
  	(Input) edge (Entry)
  	(Entry) edge (pool)
  	(pool) edge (P1)
  	(P1) edge (P2)
  	(P2) edge (P3)
  	(P3) edge (P4)
  	(P4) edge (P5)
  	(P5) edge (P6)
  	(P6) edge (P7)
  	(P7) edge (P8)
  	(P8) edge (P9)
  	(P9) edge (P10)
  	(P10) edge (Output)
  	;
\end{scope}

\end{tikzpicture}}
        \caption{Stacked block architecture of a convolutional neural network of ResNet type}
        \label{fig:resnet_architecture}
    \end{subfigure}
   \hspace{.5cm}
    \begin{subfigure}[b]{0.3\textwidth}
        \vspace{-1cm}
        \centering
        \scalebox{0.7}{\tikzset{plus/.style={circle,draw,fill=white,font=\sffamily\Large\bfseries,minimum size = 1mm}}
\tikzset{weight/.append style={rectangle,fill=yellow!10,draw,font=\sffamily\Large\bfseries}}
\tikzset{input/.append style={draw=none,font=\sffamily\Large\bfseries}}
\tikzset{pool_node/.append style={draw=none,font=\sffamily\Large\bfseries,minimum width = 50mm}}

\begin{tikzpicture}[thick,>=stealth',shorten >=1pt,auto,node distance=15mm]
    \node[input] (Input) {};
    \node[weight] (Entry) [below of=Input] {weight layer};
    \node[weight] (pool) [below of=Entry] {weight layer};
    \node[plus] (P1) [below of=pool, node distance =13mm,label={[label distance = 5mm]180:\Large $g(x)+x$}] {+};
    
    \node[pool_node] (Output) [below of=P1, node distance =13mm] {Output};

\begin{scope}[very thick,decoration={
    markings,
    mark=at position 0.55 with {\arrow{>}}}
    ] 
\draw [postaction={decorate}] 
    ($ (Input) !.4! (Entry) $) to[out=5,in=-5,looseness=2.5] (P1.center);

  \path[every node/.style={font=\sffamily\small,
  		fill=white,inner sep=1pt},
  		every edge/.append style={very thick,->,decoration={markings, mark=at position 1 with {\arrow{->}}}}
  		]
  	(pool) edge (P1);
 \end{scope}
 \begin{scope}[very thick,decoration={
    markings,
    mark=at position 1 with {\arrow{>}}}
    ] 
  	\draw [postaction={decorate}]  (Input) -- (Entry) node [midway,left, fill=white] {$\large x$};
  	\draw [postaction={decorate}]  (Entry) -- (pool) node [midway, fill=white,label={[label distance = 20mm]180:\Large \text{$\Large {g(x)}$}},label={[label distance = 12mm]-5:\begin{tabular}{c}
  	     $\large x$  \\
  	      \large \textbf{Id.}
  	\end{tabular}}] {\large ReLU};
  	\draw [postaction={decorate}]  (P1) -- (Output) node [midway, fill=white] {\large ReLU};
  	\node[circle,fill=white] at (P1.center){+} ;
\end{scope}

\end{tikzpicture}}
        \caption{Building block of ResNet with identity shortcut connection.}
        \label{step2er}
        \label{fig:resnet_block}
    \end{subfigure}
    \caption{Visualizations of CNNs and ResNet}
\end{figure}

In~\cite{lin2018resnet} it is shown that ResNets are universal approximators for functions from $\mathcal{L}_1(\R^d)$. This is of particular interest, since the width of the ResNets considered in this paper is bounded by the input dimension $d$ and it is shown in~\cite{HaninFiniteWidth,WidthNIPS} that standard feedforward neural networks with width bounded by $d$ are not universal approximators (see also Subsection~\ref{sec:Universal}). The idea of the proof relies on a reapproximation of step functions (which are dense in $\mathcal{L}_1$).

Instead of forcing skip connections to encode an identity function, one can also consider more general skip connections linking earlier layers to deeper layers with connections that are learned in the training process. Such networks have, for instance, \ been considered in~\cite{YAROTSKY2017103,guhring2019error}. Slightly better approximation rates for those architectures in contrast to the standard feedforward case are shown for functions from $F^p_{n,d}$. In more detail, when allowing skip connections, the upper approximation bounds shown in Theorem~\ref{thm:yarotsky} can be improved by dropping the square of the log terms.

\subsection{Recurrent Neural Networks}
Until now we only considered feedforward network architectures that were able to deal with a fixed input and output dimension. However, in many applications it is desirable to use a neural network that can handle sequential data with varying length as input/output. This is for example often the case for natural language processing tasks where one might be interested in predicting the topic of a sentence or text (varying input dimension, fixed output) or in producing an image caption for a given image (fixed input dimension, varying output dimension). 

Recurrent neural networks are specialized for those tasks. In its vanilla form, a recurrent neural network computes a hidden state $h_t = f_{\theta}(h_{t-1}, x_t)$ from the current input $x_t$ and, using a \emph{recurrent} connection, the previous hidden state $h_{t-1}$. The hidden state $h_t$ can then in a second step be used to compute an output $y_t=g_{\theta}(h_t)$ and enables the network to memorize features from previous inputs until time $t$, such that the output $y_t$ depends on $x_0,\ldots,x_t$ (see Figure~\ref{fig:recurrent neural network}). It it important to note that the same functions $f_\theta$ and $g_\theta$, parametrized by the weight vector $\theta$, are used in each time step. For an in-depth treatment of recurrent neural networks we refer to~\cite[Chapter 10]{Goodfellow}. 

\begin{figure}[ht]
\centering
\tikzset{weight/.append style={rectangle,fill=yellow!10,draw,font=\sffamily\Large\bfseries}}
\tikzset{input/.append style={draw=none,font=\sffamily\Large\bfseries}}
\tikzset{pool_node/.append style={draw=none,font=\sffamily\Large\bfseries,minimum width = 50mm}}

\begin{tikzpicture}[thick,>=stealth',shorten >=1pt,auto,node distance=13mm]
    \node[input] (Input){$x_t$};
    \node[weight] (Entry) [below of=Input] {$f_\theta$};
    \node[weight] (pool) [below of=Entry] {$g_\theta$};
    \node[pool_node] (Output) [below of=pool, node distance =13mm] {$y_t$};


 \begin{scope}[very thick,decoration={
    markings,
    mark=at position 1 with {\arrow{>}}}
    ] 
  	\draw [postaction={decorate}]  (Input) -- (Entry) node [midway,left, fill=white] {};
  	\path[very thick] (Entry) edge [out=45, in=315, loop]
 (Entry);
  	\draw [postaction={decorate}]  (Entry) -- (pool) node [midway, fill=white,label={[label distance = 20mm]180:\Large \text{}},label={[label distance = 12mm]-5:\begin{tabular}{c}
  	      \\
  	     \large 
  	\end{tabular}}] {\large };
  	\draw [postaction={decorate}]  (pool) -- (Output) node [midway, fill=white] {};
\end{scope}
\end{tikzpicture}
        \caption{Visualization of the architecture of a recurrent neural network}
        \label{fig:recurrent neural network}
\end{figure}
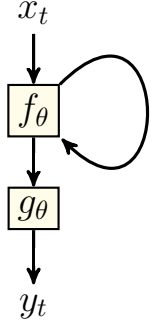

Recurrent neural networks can be viewed as dynamical systems \cite{sontag1992neural}. Hence, it is natural to ask if the universal approximation capacity of feedforward neural networks can be transferred to recurrent neural networks with respect to dynamical systems. In~\cite{sontag1992neural}, discrete-time dynamical systems of the form 
\begin{align*}
    h_{t+1}&= f(h_t, x_t),\\
    y_t &= g(h_t),
\end{align*}
where $f:\R^n\times\R^m\to\R^n, g:\R^n\to\R^p,$ and $x_t\in\R^n, h_t\in\R^m, y_t\in\R^p,$ are considered. We call the tuple $\Sigma = (n,m,p,f,g)$ a \emph{time-discrete dynamical system} and say that $n$ is the \emph{dimension of the hidden state} of $\Sigma$, $m$ its \emph{input dimension} and $p$ its \emph{output dimension}. Moreover, we call $x_0,\ldots,x_T$ the \emph{input of the system until time $T\in\N$} and $h_0$ the \emph{initial state}. With these tools at hand we can now formally define a shallow sequence-to-sequence recurrent neural network.

\begin{definition}\label{def:ShallowSequence}
Let $\Sigma = (n,m,p,f,g)$ be a dynamical system and $\sigma:\R\to\R$. Then $\Sigma$ is a \emph{shallow sequence-to-sequence recurrent neural network with activation function $\sigma$}, if, for matrices $\mathbf{A}\in\R^{n \times n}, \mathbf{B}\in\R^{m\times n}, \mathbf{C}\in\R^{ n\times p}$, the functions $f$ and $g$ are of the form
\[
 f(h,x)=\sigma(\mathbf{A}h+\mathbf{B}x)\quad\text{and}\quad g(x)=\mathbf{C}x.
\]
\end{definition}

It is then shown in \cite{sontag1992neural}, that recurrent neural networks of the form described above are able to approximate time-discrete dynamical systems (with only weak assumptions on $f$ and $g$) arbitrarily well. To formalize this claim, we first describe how the closeness of two dynamical systems is measured by the authors. For this, we introduce the following definition from the paper \cite{sontag1992neural}:
\begin{definition}
Let $\Sigma=(n,m,p,f,g)$ and $\tilde\Sigma=(\tilde n,m,p,\tilde f,\tilde g)$ be two discrete-time dynamical systems as introduced in Definition \ref{def:ShallowSequence}. Furthermore, let $K_1\subset\R^n, K_2\subset\R^m$ be compact, $T\in\N,$ and $\eps>0$. Then \emph{$\tilde \Sigma$ approximates $\Sigma$ with accuracy $\eps$ on $K_1$ and $K_2$ in time $T$,} if the following holds:

There exist continuous functions $\alpha:\R^{\tilde n}\to\R^n$ and $\beta:\R^n\to\R^{\tilde n}$ such that, for all initial states $h_0\in K_1$ and $\beta(h_0)$ for $\Sigma$ and $\tilde \Sigma$, respectively, and all inputs $x_0,\ldots,x_T\in K_2$ for $\Sigma$ and $\tilde \Sigma,$ we have
\[
\norm{h_t - \alpha(\tilde h_t)}_2<\eps\quad\text{for}\quad t=0,\ldots,T
\]
and
\[
\norm{y_t - \tilde y_t}_2<\eps\quad\text{for}\quad t=1,\ldots, T.
\]
\end{definition}

Note that while input and output dimension of $\Sigma$ and $\tilde \Sigma$ coincide, the dimension of the hidden states might differ.
Based on the universal approximation theorem for shallow feedforward neural networks (see Subsection~\ref{subsec:ShallowUniversal}), a universality result for recurrent neural networks with respect to dynamical systems has been derived. 

\begin{theorem}[\cite{sontag1992neural}]\label{thm:sontag_recurrent neural network}
Let $\sigma:\R\to\R$ be an activation function for that the universal approximation theorem holds. Let $\Sigma=(n,m,p,f,g)$ be a discrete-time dynamical system with $f,g$ continuously differentiable. Furthermore, let $K_1\subset\R^n, K_2\subset\R^m$ be compact and $T\in\N$. Then, for each $\eps>0,$ there exist $\tilde n \in\N$ and matrices $\mathbf{A}\in\R^{\tilde n \times \tilde n}, \mathbf{B}\in\R^{m\times\tilde n}, \mathbf{C}\in\R^{\tilde n\times p}$ such that the shallow recurrent neural network $\tilde \Sigma =(\tilde n,m,p,\tilde f,\tilde g)$ with activation function $\sigma$ and 
\[
\tilde f(h,x)=\sigma(\mathbf{A}h+\mathbf{B}x)\quad\text{and}\quad\tilde g(x)=\mathbf{C}x,
\]
approximates $\Sigma$ with accuracy $\eps$ on $K_1$ and $K_2$ in time $T$.
\end{theorem}

Similar results also based on the universality of shallow feedforward neural networks have been shown in~\cite{RnnsUniversalApproximators} and in~\cite{doya1993universality}. In~\cite{doya1993universality} and~\cite{polycarpou1991identification}, in addition the case of continuous dynamical systems is included.

Discrete-time fading-memory systems are dynamical systems that asymptotically ``forget'' inputs that were fed into the system earlier in time. In \cite{Matthews1993} the universality of recurrent neural networks for these types of systems was shown.

The above results cover the case of recurrent neural networks realizing sequence-to-sequence mappings. As mentioned in the introduction of this section, recurrent neural networks are also used in practice to map sequences of varying length to scalar outputs, thus (for $\R^m$-valued sequences) realizing functions mapping from $\bigcup_{i\geq 0} (\R^m)^i \to \R$. In~\cite{hammer2000approximation}, the capacity of such recurrent neural networks to approximate arbitrary measurable functions $f:\bigcup_{i\geq 0} (\R^m)^i \to \R$ with high probability is shown. Note that in contrast to approximation results for time-discrete dynamical systems, where the system to be approximated naturally displays a recursive structure (as does the recurrent neural network), no such assumption is used. It is furthermore shown that recurrent neural networks are not capable of universally approximating such functions with respect to the sup norm. 

For the rest of this section we denote by $(\R^m)^*\coloneqq\bigcup_{i\geq 0} (\R^m)^i$ the set of finite sequences with elements in $\R^m$. To present the aforementioned results in more detail, we start with a formal introduction of the type of recurrent neural network considered in the paper \cite[Def.~1]{hammer2000approximation}. 

\begin{definition}
Any function $f:\R^n\times\R^m\to\R^n$ and initial state $h_0\in\R^n$ recursively induce a mapping $\tilde f_{h_0}:(\R^m)^*\to\R^n$ as follows: 
\begin{equation*}
    \tilde f_{h_0}([x_1,\ldots,x_i])=\begin{cases}
    h_0,\quad\text{if }i=0\\
    f(\tilde f_{h_0}([x_1,\ldots, x_{i-1}]), x_i),\quad\text{else}.
    \end{cases}
\end{equation*}
A \emph{shallow sequence-to-vector recurrent neural network with initial state $h_0\in\R^n$ and activation function $\sigma:\R\to\R$} is a mapping from $(\R^m)^*$ to $\R^p$ of the form 
\[
x\mapsto \mathbf{C}\cdot \tilde f_{h_0}(x).
\]
Here, $\mathbf{C}\in\R^{n\times p}$ is a matrix and $f$ is the realization of a shallow feedforward neural network with no hidden layer and nonlinearity $\sigma$ (applied coordinate-wise) in the output layer, i.e. $f:\R^n\times\R^m\to\R^n$ is of the form 
\[
f(h,x)=\sigma(\mathbf{A}h+\mathbf{B}x)\quad\text{for}\quad h\in\R^n, x\in\R^m,
\]
where $\mathbf{A}\in\R^{n\times n}$ and $\mathbf{B}\in\R^{m\times n}$.
\end{definition}

In more simple terms, a sequence to vector recurrent neural network is a sequence-to-sequence recurrent neural network, where only the last output is used.

Next, we need the definition of approximation in probability from~\cite{hammer2000approximation}.
\begin{definition}
Let $\mathbb{P}$ be a probability measure on $(\R^m)^*$ and $f_1,f_2:(\R^m)^*\to\R^p$ be two measurable functions. Then $f_1$ \emph{approximates} $f_2$ \emph{with accuracy $\eps>0$ and probability} $\delta>0$ if 
\[
\mathbb{P}(\pabs{f_1-f_2}>\eps)<\delta.
\]
\end{definition}

The next theorem establishes the universal approximation capacity in probability of recurrent neural networks for real valued functions with real valued sequences of arbitrary length as input. 
\begin{theorem}[\cite{hammer2000approximation}]\label{thm:hammer_recurrent neural network}
Let $\sigma:\R\to\R$ be an activation function for that the universal approximation theorem holds and $\eps,\delta>0$. Then every measurable function $f:(\R^m)^*\to\R^n$ can be approximated with accuracy $\eps$ and probability $\delta$ by a shallow recurrent neural network.
\end{theorem}
Interestingly, even if Theorem~\ref{thm:sontag_recurrent neural network} only deals with approximations of dynamical systems in finite time, it can still be utilized for a proof of Theorem~\ref{thm:hammer_recurrent neural network}. \cite{hammer2000approximation} outlines a proof based on this idea\footnote{The actual proof given in~\cite{hammer2000approximation} is not based on Theorem~\ref{thm:sontag_recurrent neural network}, which enables the authors to draw some conclusions about the complexity necessary for the recurrent neural network in special situations.}, which we will depict here to underline the close connections of the theorems. The main ingredient is the observation that $\mathbb{P}(\bigcup_{i>i_0}(\R^m)^i)$ converges to zero for $i_0\to\infty$. As a consequence, one only needs to approximate a function $f:(\R^m)^*\to\R^n$ on sequences up to a certain length in order to achieve approximation in probability. The first step of the proof consists now of approximating $f$ by the output of a dynamical system for input sequences up to a certain length. In the second step one can make use of Theorem~\ref{thm:sontag_recurrent neural network} to approximate the dynamical system by a recurrent neural network.

In \cite{khrulkov2017expressive} the expressive power of certain recurrent neural networks, where a connection to a tensor train decomposition can be made, is analyzed.

Network architectures dealing with more structured input like trees or graphs that make also use of recurrent connections and can thus be seen as a generalization of recurrent neural networks have been considered in~\cite{hammer1999approximation}, \cite[Ch.~3]{hammer2000learning}, \cite{bianchini2005recursive} and \cite{scarselli2008computational}. 


\section*{Acknowledgments}
The authors would like to thank Philipp Petersen for valuable comments for improving this manuscript. Moreover, they would like to thank Mark Cheng for creating most of the figures and Johannes von Lindheim for providing Figure~\ref{fig:SwissRoll}. 

I.G.\ acknowledges support from the Research Training Group "Differential Equation- and Data-driven Models in Life Sciences and Fluid Dynamics: An Interdisciplinary Research Training Group (DAEDALUS)" (GRK 2433) funded by the German Research Foundation (DFG). G.K.\ acknowledges partial support by the Bundesministerium f\"ur Bildung und Forschung (BMBF) through the Berliner Zentrum for Machine Learning (BZML), Project AP4, by the Deutsche Forschungsgemeinschaft (DFG) through grants CRC 1114 ``Scaling Cascades in Complex Systems'', Project B07, CRC/TR 109 ``Discretization in Geometry and Dynamics'', Projects C02 and C03, RTG DAEDALUS (RTG 2433), Projects P1 and P3, RTG BIOQIC (RTG 2260), Projects P4 and P9, and SPP 1798 ``Compressed Sensing in Information Processing'', Coordination Project and Project Massive MIMO-I/II, by the Berlin Mathematics Research Center MATH+, Projects EF1-1 and EF1-4, and by the Einstein Foundation Berlin.



\end{document}